\documentclass{article}

\usepackage{arxiv}

\usepackage[utf8]{inputenc}
\usepackage[T1]{fontenc}
\usepackage{hyperref}
\hypersetup{
  hidelinks,
  pdftitle={MSC-CMA-ES: Structure-Aware Restarts for CMA-ES via Cyclic Nearest-Better Basin Discovery},
  pdfsubject={cs.NE, cs.AI},
  pdfauthor={Dimitar Nedanovski, Svetoslav Nenov, Dimitar Pilev},
  pdfkeywords={CMA-ES, basin of attraction, nearest-better clustering, restart strategy, CEC2014, CEC2017, CEC2020, CEC2022},
}
\usepackage{url}
\usepackage{booktabs}
\usepackage{amsfonts}
\usepackage{amsmath}
\usepackage{amsthm}
\usepackage{algorithm}
\usepackage{algpseudocode}
\usepackage[numbers]{natbib}
\usepackage{doi}
\usepackage{graphicx}
\usepackage{comment}
\usepackage{multirow}
\usepackage{amssymb}
\usepackage{subcaption}

\theoremstyle{definition}
\newtheorem{definition}{Definition}
\theoremstyle{remark}
\newtheorem{remark}{Remark}

\theoremstyle{plain}

\newtheorem{lemma}{Lemma}

\usepackage{placeins}

\newcommand{\R}{\mathbb{R}}

\newcommand{\norm}[1]{\left\lVert #1 \right\rVert}

\DeclareMathOperator*{\argmin}{argmin}
\newcommand{\todo}[1]{}

\title{MSC-CMA-ES: Structure-Aware Restarts for CMA-ES\\
via Cyclic Nearest-Better Basin Discovery}
\renewcommand{\shorttitle}{MSC-CMA-ES}

\date{July, 2026}

\author{%
  \href{https://orcid.org/0009-0001-7212-9167}{\includegraphics[scale=0.06]{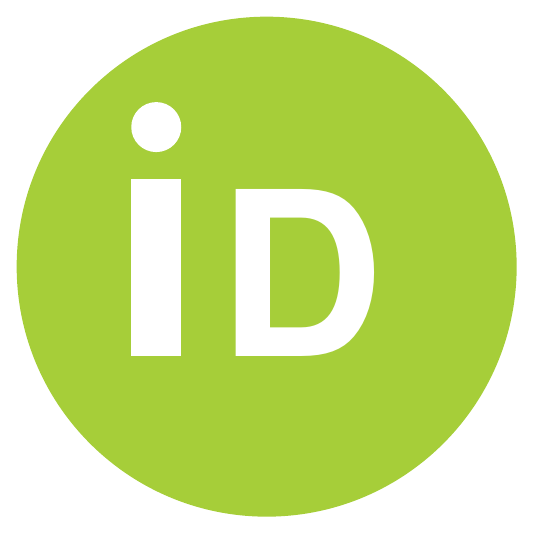}\hspace{1mm}Dimitar Nedanovski}\\
  Faculty of Mathematics and Informatics\\
  Sofia University St.\ Kliment Ohridski\\
  James Baucher Blvd., 1164 Sofia, Bulgaria\\
  \texttt{dnedanovski@gmail.com}
  \And
  \href{https://orcid.org/0000-0002-1444-8522}{\includegraphics[scale=0.06]{orcid-eps-converted-to.pdf}\hspace{1mm}Svetoslav Nenov}\\
  Department of Mathematics\\
  University of Chemical Technology and Metallurgy\\
  Kl. Ohridski, Blvd., 1797, Sofia, Bulgaria\\
  \texttt{nenov@uctm.edu}
  \And
  Dimitar Pilev\\
  Department of Informatics\\
  University of Chemical Technology and Metallurgy\\
  Kl. Ohridski, Blvd., 1797, Sofia, Bulgaria\\
  \texttt{pilev@uctm.edu}
}

\hypersetup{
  pdftitle={MSC-CMA-ES: Structure-Aware Restarts for CMA-ES via Cyclic Nearest-Better Basin Discovery},
  pdfsubject={cs.NE, cs.AI},
  pdfauthor={Dimitar Nedanovski, Svetoslav Nenov, Dimitar Pilev},
  pdfkeywords={CMA-ES, basin of attraction, nearest-better clustering, restart strategy, CEC2014, CEC2017, CEC2020, CEC2022},
}

\begin{document}
\maketitle

\begin{abstract}
CMA-ES behaves, per restart, primarily as a local optimizer; multimodal search relies on
restart strategies such as IPOP and BIPOP, which draw every restart
uniformly and reuse no information from previous evaluations. Multi-Start Clustering CMA-ES (MSC-CMA-ES) makes restarts structure-aware: in alternating cycles, a
Sobol pre-sample is partitioned into approximate basins of attraction
by nearest-better clustering, restarts are seeded basin by basin with
locally scaled step-sizes and population sizes, redundant basin
visits are detected and excluded, and the remaining budget is spent
on a budget-bounded, tolerance-disabled local refinement of the best-so-far solution.
We evaluate the method on the four CEC suites (2014, 2017, 2020, and 2022) at their official budgets, across ten (suite, dimension) cells with dimensions 5--30, with 51 runs per function, and compare it with BIPOP-CMA-ES and five differential-evolution algorithms (ARRDE, jSO, j2020, NL-SHADE-RSP, and L-SRTDE).
Read per function class over the $D\leq 20$ cells, the results split three ways.
On composition functions MSC-CMA-ES attains the best value on all aggregate measures, with $2.7\times$ the fixed-budget target coverage of BIPOP-CMA-ES --- the highest composition coverage of any algorithm evaluated --- and leads in most suites, dimensions, and budgets. On unimodal and simple-multimodal functions it attains the best median error in aggregate but the lowest deep-target coverage, with orderings varying by suite and dimension. On hybrid functions both CMA algorithms trail the leading DE algorithms. All results and scripts are publicly available.
\end{abstract}


\section{Introduction}\label{sec:intro}

The Covariance Matrix Adaptation Evolution Strategy
(CMA-ES)~\cite{Hansen2001} adapts a single multivariate Gaussian and is
therefore, per run, a local optimizer: on a multimodal landscape it
converges to one basin of attraction.
Multi-basin search is traditionally delegated to restart strategies, with IPOP (increasing-population) \cite{AugerHansen2005} and BIPOP (bi-population) \cite{Hansen2009} restart strategies serving as the de facto standards. Even two decades after its introduction, BIPOP-CMA-ES remains one of the most robust general-purpose baselines for continuous black-box optimization.
Both
strategies, however, leave information unused: every restart draws
$x_0$ uniformly and sets $\sigma_0=(u-l)/4$ (the default initial step-size), so the topographic
information already acquired on previous evaluations never informs the
next restart.  On multi-basin landscapes this has two recurring costs:
evaluations are spent contracting a domain-scale $\sigma_0$ onto narrow
local structure, and successive restarts re-enter basins that earlier
runs have already exhausted.  Composition-type benchmark functions are
constructed by combining rotated and shifted base functions around a
randomly located global optimum and several deep, randomly located
local optima with differing local
properties~\cite{Liang2014cec2014,AwadEtAl2017,SugathanCEC2020,TornViitanen1992}; a
uniformly initialized restart has no mechanism to prefer the global
basin over the deceptive ones.

We introduce Multi-Start Clustering CMA-ES (MSC-CMA-ES), which makes restarts structure-aware. The search proceeds in cycles. Each cycle opens with Phase~0: a Sobol design of $M = 2^n$ points is evaluated. These points are then partitioned into approximate basins of attraction by Nearest-Better Clustering~\cite{Preuss2010,Preuss2012}. The cutting threshold is selected automatically by a staircase scan of the nearest-better edge-length distribution.

Phase~1 seeds one CMA-ES restart per basin, moving from the smallest to the largest. Each restart initializes from the basin's best sampled point. The initial step-size $\sigma_0$ is computed from the coordinate-wise spread of the basin around its elite centre. The initial population size is proportional to the basin's sample size. A $k$-NN membership vote identifies the basin into which each restart converges. Any basin resolved twice is excluded from further restarts within that cycle. 

Successive cycles alternate between two complementary configurations of this pipeline. The first configuration partitions the sample into many small basins. These are probed by short, small-population runs. The second configuration partitions the space into few large basins. These are searched by long, large-population runs. 

When the remaining budget no longer supports a full cycle, a final refinement stage begins. This stage spends all remaining evaluations on a single CMA-ES run from the incumbent solution. All
tolerance-based convergence criteria are disabled for this run. Consequently, the remaining budget---and not a tolerance floor---terminates the optimization. The full algorithm is specified in Section~\ref{sec:algorithm}.

In Section~\ref{sec:experiments} we report the official-budget results
of seven algorithms---MSC-CMA-ES, BIPOP-CMA-ES, and five
differential-evolution baselines (ARRDE~\cite{ARRDE2023},
jSO~\cite{BrestEtAl2017}, j2020~\cite{Brest2020j2020},
NL-SHADE-RSP~\cite{Stanovov2021nlshadersp}, and
L-SRTDE~\cite{Stanovov2024lsrtde})---read along four cross-sections:
\begin{enumerate}
\item \textbf{By function class} (Table~\ref{tab:class-summary}): on
        the composition class MSC-CMA-ES attains the lowest summed mean,
        median, and best error, and the highest target coverage (FBTC, a fixed-budget analogue of the COCO-style target-hit profile; Section~\ref{sec:protocol})---%
        $2.7\times$ that of the BIPOP-CMA-ES baseline.
\item \textbf{By dimension} (Table~\ref{tab:dim-summary}): the lowest summed mean and median at every $D\le20$, the margin over the nearest baseline ranging from $6\%$ at $D = 10$,
$10.2\%$ at $D = 20$, to $53\%$ at $D = 5$.
\item \textbf{Within one suite, across dimensions} (Fig.~\ref{fig:fbtc-comp}): on the
      CEC2020 composition class MSC-CMA-ES attains the highest FBTC at
      every $D\in\{5,10,15,20\}$, the margin widening with dimension.

\item \textbf{Per cell} (Fig.~\ref{fig:rank-comparison-all}): on the
      composition class the error-magnitude ranking (worst-, median-, and
      best-SUM) that places MSC-CMA-ES at the front is preserved across
      cells, whereas target coverage (FBTC) reorders the field and does
      not track the magnitude measures---on CEC2014 $D = 10$
      (Fig.~\ref{fig:rank-2014d10}) it drops MSC-CMA-ES to third. The
      aggregate composition-FBTC lead of point~1
      (Table~\ref{tab:class-summary}) therefore coexists with per-cell reorderings rather than holding cell by cell.
\end{enumerate}
Section~\ref{sec:experiments} also reports the three regimes where this
advantage does not hold. On the unimodal and
simple-multimodal functions taken together, the outcome is a trade-off:
MSC-CMA-ES attains the lowest summed median error ($138.79$) but the
lowest summed FBTC of the field ($25.46$, against $27.04$--$30.63$ for
the baselines). On the \emph{hybrid} class the deficit
is several-fold: the summed mean and median ($265.11$/$294.35$) trail the
strongest DE baselines ($39.89$/$33.26$), and BIPOP-CMA-ES trails
alongside, so the gap is shared by both CMA variants rather than
introduced by the restart layer. At $D = 30$, outside the $D\le20$ design
envelope, MSC-CMA-ES is mid-pack: across the five convergent algorithms
the summed mean spans roughly a factor of two, while j2020 and
NL-SHADE-RSP diverge on a subset of functions by over an order of
magnitude.

Finally, Section~\ref{sec:budget} reports how the advantage scales
with the evaluation budget, and Section~\ref{sec:discussion} examines
why, since the Phase-0 sample is a fixed up-front
cost that only later restarts amortize.
On the CEC2020 composition class
the coverage margin widens monotonically with both dimension and budget
(Fig.~\ref{fig:fbtc-comp}): at $D = 20$ the FBTC rises from $\approx1.1$
at $2\times10^{6}$ evaluations to $2.6$ (of a maximum $3$) at
$4\times10^{7}$, while no baseline exceeds $1.2$.
The $D = 30$ loss reported in Section~\ref{sec:experiments} is correspondingly
budget-limited rather than intrinsic: on the CEC2017 $D = 30$ composition class
(Fig.~\ref{fig:d30-comp-budget}) MSC-CMA-ES is mid-pack at the official
$3\times10^{5}$ budget but attains the lowest worst-, mean-, and best-SUM at
$10^{6}$ (Section~\ref{sec:discussion}).

The complete source code for 
MSC-CMA-ES, implemented in Python and optimized for the Intel Distribution for Python, 
is provided in the Supplementary Material\footnote{\url{https://github.com/snenovgmailcom/cma_es_project}} alongside the execution scripts. This material also contains all computational results.

\subsection*{Contributions}
\begin{enumerate}
  \item A \textbf{staircase $\phi$ selection procedure} for NBC that
        determines the cutting threshold automatically from the
        edge-length distribution of the nearest-better tree,
        eliminating manual $\phi$ tuning.
  \item \textbf{Per-basin step-size and population-size adaptation},
        replacing the fixed BIPOP heuristic $\sigma_0=(u-l)/4$ with a
        $\sigma_0$ computed from the discovered basin geometry and a
        CMA population size proportional to the basin's sample size.
  \item A \textbf{convergence-tracking mechanism} ($k$-NN majority
        vote) that detects redundant basin visits and excludes
        twice-resolved basins from further restarts within the cycle.
  \item An \textbf{alternating two-configuration cycle schedule} that
        re-clusters the same landscape at two granularities---many
        small basins probed by short runs, few large basins searched
        by long runs---with \textbf{zero-cost Phase-0 reuse} of the
        sample across cycles via the nested-prefix property of Sobol
        sequences.
\end{enumerate}

\section{The MSC-CMA-ES algorithm}\label{sec:algorithm}

\subsection{Setting and notation}

We minimize a continuous function $f:\Omega\to\R$ on a box
$\Omega=\prod_{j=1}^{D}[l_j,u_j]\subset\R^{D}$ under a fixed budget of
$T$ function evaluations.  Distances are measured in normalized
coordinates, $\hat x_j=(x_j-l_j)/(u_j-l_j)$, so that
$\hat\Omega=[0,1]^D$.

The method runs a sequence of \emph{cycles} and closes with a
single refinement.  Each cycle has two phases: \emph{Phase~0}
(Section~\ref{sec:phase0}) evaluates a space-filling sample of $\Omega$
and partitions it into \emph{basins} by nearest-better clustering;
\emph{Phase~1} (Section~\ref{sec:phase1}) then launches one CMA-ES
restart per basin, in ascending order of basin size.  Cycles repeat
while a budget gate (Section~\ref{sec:cycles}) admits another one; once
the main budget is spent, a final CMA-ES \emph{refinement}
(Section~\ref{sec:refinement}) consumes whatever evaluations remain,
polishing the incumbent $x^{\mathrm{best}}$.

The two configurations alternate strictly by cycle parity---cycle $c$ uses the parameters $\theta_{\mathrm{C}}$ for $c$ even and $\theta_{\mathrm{B}}$ for $c$ odd.
Both are universal across dimensions and budgets except the CMA step-size
divisor $\delta_{\mathrm{ref}}$, which is anchored at $D = 10$ and rescaled
by $\sqrt{10/D}$ (Table~\ref{tab:configs}). Algorithm~\ref{alg:overview}
gives the skeleton, Algorithm~\ref{alg:msc} the full version.

\begin{algorithm}[t]
\caption{MSC-CMA-ES (skeleton; full version in
  Algorithm~\ref{alg:msc})}
\label{alg:overview}
\begin{algorithmic}[1]
\Require objective $f$, box $\Omega$, budget $T$, configuration pair
  $(\theta_{\mathrm{C}},\theta_{\mathrm{B}})$
\State reserve refinement budget:\
  $T_{\mathrm{main}}\gets T-\lfloor r_{\mathrm{C}}T\rfloor$
\For{cycle $c=0,1,2,\dots$ while the budget gate
  (Sec.~\ref{sec:cycles}) holds}
  \State $\theta\gets\theta_{\mathrm{C}}$ if $c$ even else
    $\theta_{\mathrm{B}}$ \Comment{alternating configurations}
  \State \textbf{Phase~0:} sample $\Omega$, cluster into basins
    $B_1\le\cdots\le B_K$
    \Comment{Sec.~\ref{sec:phase0}; odd cycles reuse the previous sample}
  \State \textbf{Phase~1:} one CMA-ES restart per basin, smallest first
    \Comment{Sec.~\ref{sec:phase1}}
\EndFor
\If{$T - \mathrm{nfev} \ge 10D$}
  \State \textbf{Refinement:} single CMA-ES from $x^{\mathrm{best}}$ on
    the remaining budget \Comment{Sec.~\ref{sec:refinement}}
\EndIf
\State \Return incumbent $x^{\mathrm{best}}$
\end{algorithmic}
\end{algorithm}

\subsection{Phase 0: basin discovery}\label{sec:phase0}

\paragraph{Design.}
Phase~0 evaluates a Sobol sample $P=\{x_1,\dots,x_M\}\subset\Omega$ with $M = 2^n$ and records $f_i=f(x_i)$, $i=1,\dots,M$.

\paragraph{Nearest-better clustering.}
Every sampled point is connected to its nearest \emph{better}
neighbour, in normalized coordinates:
\begin{equation}\label{eq:nbtree}
\mathrm{nb}(i) =
\begin{cases}
\displaystyle\min\Bigl\{\argmin_{j:\,f_j<f_i}\ \norm{\hat x_j-\hat x_i}\Bigr\}, & \{j: f_j<f_i\}\neq\varnothing,\\[8pt]
-1, & \text{otherwise,}
\end{cases}
\qquad
\ell_i = \begin{cases}
\norm{\hat x_{\mathrm{nb}(i)}-\hat x_i}, & \mathrm{nb}(i) \geq 0,\\[8pt]
0, & \text{otherwise.}
\end{cases}
\end{equation}
The sample-best point has no better neighbour and becomes the root;
the result is a tree on the $M$ sample points with $M-1$ edges.
Relation~\eqref{eq:nbtree} is evaluated with incremental $k$-NN queries
(cap $k_{\max}=256$); if no better point appears
within the cap, the edge falls back to the sample-best point.

We follow the two heuristic rules of Preuss~\cite{Preuss2010,Preuss2012} to build the NBC partition on top of this tree.

\paragraph{Rule~1 (long-edge cut).}
Let
\begin{equation*}
\bar{\ell} = \frac{1}{M-1} \sum_{i:\,\mathrm{nb}(i)\neq -1} \ell_i
\end{equation*}
be the mean edge length.  Cut all edges exceeding a multiple of the
mean:
\begin{equation}\label{eq:nbc-rule1}
\text{cut edge } (i, \mathrm{nb}(i))
  \quad\text{if}\quad \ell_i > \phi \, \bar{\ell},
\end{equation}
where $\phi > 0$ is the cutting threshold.  Point~$i$ whose edge is
cut becomes a new basin root.

\paragraph{Rule~2 (hub detection).}
A point $j$ with high indegree whose own outgoing edge is much longer
than its incoming edges is a local attractor incorrectly connected to
a distant parent.  Let
$\mathrm{in}(j) = \{i : \mathrm{nb}(i) = j\}$ in the graph pruned by
Rule~1, and
$\bar{\ell}_{\mathrm{in}}(j) = |\mathrm{in}(j)|^{-1}
\sum_{i \in \mathrm{in}(j)} \ell_i$.
Then:
\begin{equation}\label{eq:nbc-rule2}
\text{cut edge } (j, \mathrm{nb}(j))
  \quad\text{if}\quad
  |\mathrm{in}(j)| \ge n_{\min}
  \;\text{ and }\;
  {\ell_j} > b{\bar{\ell}_{\mathrm{in}}(j)},
\end{equation}
where $n_{\min}$ is the minimum indegree and $b > 0$ is the ratio
threshold, a component of the tuned configuration
(Table~\ref{tab:configs}).

After both rules are applied, basin labels are obtained by pointer jumping
on the pruned parent array.
Let $\mathcal{T}:\{1,\dots,M\}\to\{1,\dots,M\}$
be the pruned-parent map, with $\mathcal{T}(i)=\mathrm{nb}(i)$ on the edges
retained after Rules~1--2 and $\mathcal{T}(i)=i$ on the roots (points whose
edge was cut, and the NB-tree root itself).
Let $\mathcal{T}^{k}$ denote the $k$-fold composition of $\mathcal{T}$
with itself, $\mathcal{T}^{k+1} = \mathcal{T}\circ \mathcal{T}^{k}$. Since $\mathcal{T}$ fixes the roots and every point reaches its root in finitely many steps, there is a smallest $d$ with $\mathcal{T}^{k}=\mathcal{T}^{d}$ for all $k\ge d$, and
$\mathrm{root}(i)=\mathcal{T}^{d}(i)$. Pointer jumping computes this
fixed point by repeated squaring,
\[
  \mathcal{R}_{0}=\mathcal{T},\qquad
  \mathcal{R}_{t+1}=\mathcal{R}_{t}\circ\mathcal{R}_{t},
\]
so that $\mathcal{R}_{t}=\mathcal{T}^{2^{t}}$: each iteration doubles
the composed power and thus halves the remaining path length to the
root. Hence $\mathcal{R}_{t}=\mathrm{root}$ as soon as $2^{t}\ge d$,
i.e. after $\lceil\log_{2}d\rceil\le\lceil\log_{2}M\rceil$ iterations
(the worst case $d=M-1$ is a single nearest-better chain).
The resulting partition is
$\mathcal{S}=\mathcal{C}_1\sqcup\cdots\sqcup\mathcal{C}_B$, where basins
smaller than the minimum size $s_{\min}$ (Table~\ref{tab:configs}) are
discarded as noise.

\subsection{Staircase $\phi$ selection}\label{sec:staircase}

The threshold~$\phi$ is the main free parameter of
NBC~\cite{Preuss2010}, and its choice is known to be
landscape-dependent.  Rather than using a fixed~$\phi$, we
automatically select it to achieve a target number of basins
$n_{\mathrm{target}}$.

As $\phi$ decreases from $+\infty$ to $0$, the Rule-1 cut set changes
only when $\phi$ crosses a ratio $\ell_i / \bar{\ell}$.  Sort these
ratios in non-increasing order to obtain the candidate ladder
\begin{equation}\label{eq:phi-stars}
  \phi_k^* = {\ell_{(k)}}{\bar{\ell}}^{-1},
  \qquad
  \phi_1^* \ge \phi_2^* \ge \cdots \ge \phi_{M-1}^*,
\end{equation}
where $\ell_{(1)} \ge \cdots \ge \ell_{(M-1)}$ are the sorted NB-tree
edge lengths and $\bar{\ell}$ their mean.

\begin{lemma}\label{lem:rule1-mono}
Assume the edge lengths are pairwise distinct.  Let $n_b(\phi)$ be
the number of components after Rule~1.  Then $n_b$ is non-increasing
in $\phi$, piecewise constant with jumps exactly at the $\phi_k^*$,
and $n_b(\phi) = k+1$ on the open interval
$(\phi_{k+1}^*, \phi_k^*)$.
\end{lemma}

\begin{proof}
For $\phi\in(\phi_{k+1}^*,\phi_k^*)$ the cut set
$\{i:\ell_i>\phi\bar\ell\}$ consists exactly of the edges of rank
$1,\dots,k$.  The NB graph is a tree on $M$ vertices with $M-1$
edges, so deleting $k$ of them leaves exactly $k+1$ connected
components.  Monotonicity and the jump locations follow since the
cut set grows by one edge each time $\phi$ crosses a $\phi_k^*$ from
above.  With ties, consecutive intervals merge and the statement
holds with multiplicities.
\end{proof}

\paragraph{Algorithm.}
Let $\phi^*_1 \ge \phi^*_2 \ge \cdots \ge \phi^*_{M-1}$ denote the
non-root edge ratios $\ell_i/\bar\ell$ sorted in non-increasing order
(the candidate ladder of Eq.~\ref{eq:phi-stars}), and define 
\[
  \phi^{(k)} = \tfrac12\bigl(\phi^*_k + \phi^*_{k+1}\bigr),
  \qquad k = 1,\dots,M-2 .
\]
At each $\phi^{(k)}$ we apply Rule~1 and Rule~2 to the cached NB-tree,
pointer-jump the pruned parent array to labels, and record the count of
useful basins
\[
  n_{\mathrm{useful}}\bigl(\phi^{(k)}\bigr)
  =  \bigl|\{\, a : |\mathcal{C}_a(\phi^{(k)})| \ge s_{\min} \,\}\bigr| ,
\]
where $\mathcal{C}_a(\phi)$ are the components induced at threshold
$\phi$; this scan incurs no objective evaluations. The selected
threshold is 
\[
  k^* = \min\bigl\{\, k \in \{1,\dots,M-2\} :
  n_{\mathrm{useful}}(\phi^{(k)}) \ge n_{\mathrm{target}} \,\bigr\},
  \qquad \phi_{\mathrm{used}} = \phi^{(k^*)} .
\]
When this set is empty we fall back to
$\phi_{\mathrm{used}} = \tfrac12\,\phi^*_{M-1}$, the finest Rule-1
pruning (every non-root edge cut); basins below $s_{\min}$ are then
discarded as usual.

By Lemma~\ref{lem:rule1-mono} the raw Rule-1 component count is
monotone in $\phi$ with unit jumps exactly at the $\phi^*_k$, so the
ladder enumerates every distinct Rule-1 pruning. The composed map
$\phi \mapsto n_{\mathrm{useful}}(\phi)$, however, is \emph{not}
monotone: the $s_{\min}$ filter removes components as they shrink below
$s_{\min}$ when $\phi$ decreases, and Rule~2 cuts on the already-pruned
in-degrees.
Because $\phi \mapsto n_{\mathrm{useful}}(\phi)$ is not monotone, a
bisection may miss the target; a single descending pass is used instead.
Scanning $\phi$ from large to small and returning at the first success
yields the largest ladder threshold attaining $n_{\mathrm{target}}$.

\begin{algorithm}[t]
\caption{Staircase $\phi$ selection}
\label{alg:staircase}
\begin{algorithmic}[1]
\Require NB-tree parent array, edge lengths, mean~$\bar{\ell}$,
  target~$n_{\mathrm{target}}$, min basin size~$s_{\min}$
\Ensure $\phi_{\mathrm{used}}$, history $[(\phi, n)]$
\State Sort ratios $\ell_i / \bar{\ell}$ descending:
  $\phi_1^* \ge \phi_2^* \ge \cdots \ge \phi_{M-1}^*$
\For{$k = 1, 2, \ldots, M-2$}
  \State $\phi \gets (\phi_k^* + \phi_{k+1}^*) / 2$
  \State Apply Rule~1 + Rule~2 at~$\phi$ \Comment{no new evals}
  \State Count basins with $|\mathcal{C}_a| \ge s_{\min}$
  \State Record $(\phi, n_{\mathrm{useful}})$
  \If{$n_{\mathrm{useful}} \ge n_{\mathrm{target}}$}
    \State \Return $\phi$
  \EndIf
\EndFor
\State \Return $\phi_{M-1}^* / 2$ \Comment{fallback if target unreachable}
\end{algorithmic}
\end{algorithm}

The proposed procedure applies Rule~1 and Rule~2
cuts to the stored raw NB-tree without any new function evaluations:
it operates purely on the cached parent array and edge lengths from
the initial NBC construction.  The staircase may call this routine
hundreds of times during the $\phi$ scan, but each call is
$O(M\log M)$ (vectorized pointer jumping), making the total cost
negligible compared to the $M$ function evaluations.

\subsection{Phase 1: structure-aware restarts}\label{sec:phase1}

Basins are processed in ascending order of size.  For each basin $B$
not yet excluded (see below), one CMA-ES~\cite{Hansen2001} restart is
launched with three quantities derived from the basin geometry:

\paragraph{Initial point.}
$x_0=\argmin_{x\in B} f(x)$, the best sampled point of the basin.

\paragraph{Initial step-size.}
Let $c$ be the \emph{elite centre}, the mean of the best
$\max\{1,\lfloor\varepsilon|B|\rfloor\}$ points of $B$ by $f$-value, and let
\begin{equation}
  s_j = Q_{75}\bigl(\,|x_{ij}-c_j| : x_i\in B\,\bigr),  \qquad j=1,\dots,D
\end{equation}
be the per-axis upper-quartile spread around it.  Then
\begin{equation}
  \sigma_0
  =
  \max\!\Bigl(\;
     \frac{\operatorname{median}_j\, s_j}{\delta(D)},\;
     1\Bigr),
  \qquad
  \delta(D)=\delta_{\mathrm{ref}}\sqrt{10/D},
  \label{eq:sigma0}
\end{equation}
where $\delta_{\mathrm{ref}}$ is the configuration's divisor anchored
at $D=10$ and the $\sqrt{10/D}$ factor is the standard CMA box-scaling
law; the lower bound $1$ is a safety net against degenerate basins.

\paragraph{Population size.}
The population size is proportional to the basin's sample size, bounded
below by the CMA-ES default and above by a configuration-dependent value:
\begin{equation}\label{eq:popsize}
  \lambda
  = \min\!\bigl(\max(\lambda_H,\ \lceil \rho\,|B| \rceil),\
                \max(\lambda_H,\ \lambda_{\max})\bigr),
  \qquad
  \lambda_H = 4+\lfloor 3\ln D\rfloor,
\end{equation}
where $\rho$ is the (configuration-dependent) proportionality factor,
$\lambda_{\max}$ the configuration-dependent upper bound
(Table~\ref{tab:configs}), and $\lambda_H$ the default CMA-ES population
size. When $\lambda_{\max}<\lambda_H$ (the C~configuration at $D\ge 30$)
the default $\lambda_H$ prevails over the upper bound.

\begin{remark}
The two bounds in Eq.~\eqref{eq:popsize} can invert. The lower bound
$\lambda_H=4+\lfloor 3\ln D\rfloor$ grows with dimension
($\lambda_H=10,12,14,15$ at $D=10,20,30,50$), whereas the upper bound
$\lambda_{\max}^{\mathrm{C}}=12$ for the C~configuration is fixed; hence
$\lambda_H>\lambda_{\max}^{\mathrm{C}}$ for $D\ge 30$. In that regime the
CMA-ES default $\lambda_H$ overrides the upper bound, so that
C~restarts use $\lambda=\lambda_H$ rather than the nominal
$\lambda_{\max}^{\mathrm{C}}$. This is deliberate: at the fixed $D = 10$
tuning, the small upper bound $\lambda_{\max}^{\mathrm{C}}$ starves the
covariance update in high dimension, where a sample of at least
$\lambda_H$ is needed for a stable rank-$\mu$ estimate. Writing the
population as $\min(\max(\lambda_H,\lceil\rho|B|\rceil),\,
\max(\lambda_H,\lambda_{\max}))$ rather than a plain
$\operatorname{clamp}$ encodes this precedence, keeping the lower bound
active even when it exceeds the upper bound.
\end{remark}

\paragraph{Stopping.}
A restart stops at the first of: (i) an internal CMA-ES criterion
with tolerances
$\texttt{tolfun}=10^{-\tau_f}$, $\texttt{tolx}=10^{-\tau_x}$;
(ii) absolute fitness convergence of the current population,
$\max(F)-\min(F)<s_{\mathrm{tol}}$;
(iii) budget exhaustion.

\paragraph{Convergence tracking and exclusion.}
Basin membership of any point $x$ is decided by a $k$-NN majority
vote ($k=5$) over the labelled Phase-0 sample.  Before a restart is
launched, its $x_0$ is probed and the restart is skipped if the vote
maps it into an excluded basin.  After a restart finishes, the vote
of its best point identifies the basin it converged into; a basin
that has received two convergences is excluded from all further
restarts of the current cycle.  Exclusion is reset at the start of
each cycle: the next cycle re-clusters the landscape at a
different granularity, so basin identities do not carry over.

\subsection{Cycles, alternating configurations, and sample reuse}
\label{sec:cycles}

The pair (Phase~0, Phase~1) constitutes one \emph{cycle}.  Cycles
repeat with alternating configurations,
$\theta_{\mathrm{C}},\theta_{\mathrm{B}},\theta_{\mathrm{C}},\dots$:
the C~configuration partitions the sample into many small basins
probed by short, small-population runs
($n_b=25$, $s_{\min}=7$, $\lambda_{\max}=12$), the B~configuration
into few large basins searched by long, large-population runs
($n_b=5$, $s_{\min}=155$, $\lambda_{\max}=236$); all values in
Table~\ref{tab:configs}.

Since both configurations draw the same Sobol design size $M$, every
second (odd) cycle reuses the points and $f$-values cached from the
previous cycle. We re-cluster them under the odd cycle's own
configuration (in particular its own $b$, $s_{\min}$, and staircase
target), spending \emph{zero} additional evaluations on Phase~0.
A new cycle is started only if the remaining main budget is at least
$50D$ when the cycle is reuse-eligible, and $M+50D$ otherwise.

The C/B alternation plays, at the cycle level, the role of BIPOP's small- and large-population regimes; the mechanism itself is not BIPOP --- each restart is a single CMA-ES run, and regime selection follows cycle parity rather than budget balancing.

\subsection{Final refinement}\label{sec:refinement}

A fraction $r_C$ of the total budget is reserved at the start
($T_{\mathrm{main}}=T-\lfloor rT\rfloor$); the reservation is a floor,
not an allocation, since the refinement stage spends everything that
remains.  When the cycle loop terminates and at least $10D$
evaluations remain, a single CMA-ES run is started from the 
$x^{\mathrm{best}}$ with
$\sigma_0=\min\!\bigl(s^\star,\,\operatorname{median}_j(u_j-l_j)/100\bigr)$,
where $s^\star$ is the final step-size of the restart that produced
$x^{\mathrm{best}}$ (fallback:
$s^\star=\operatorname{median}_j(u_j-l_j)/\delta(D)$),
$\lambda=\max(\lambda_H,10)$, and \emph{all} tolerance-based stopping
criteria disabled
($\texttt{tolfun}=\texttt{tolx}=\texttt{tolfunhist}=0$, iteration and
stagnation caps lifted).  The run therefore terminates on the budget
or on floating-point no-effect conditions only: the budget, not a
tolerance, ends the optimization.

\begin{algorithm}[t]
\caption{MSC-CMA-ES}
\label{alg:msc}
\begin{algorithmic}[1]
\Require $f$, box $\Omega$, budget $T$, schedule
         $(\theta_{\mathrm{C}},\theta_{\mathrm{B}})$
\State $T_{\mathrm{main}} \gets T-\lfloor r_{\mathrm{C}} T\rfloor$;\quad
       $t\gets 0$;\quad cycle $c\gets 0$
\While{$t < T_{\mathrm{main}}$ \textbf{and} budget gate
       (Sec.~\ref{sec:cycles}) holds}
  \State $\theta\gets\theta_{\mathrm{C}}$ if $c$ even else
         $\theta_{\mathrm{B}}$
  \If{$c$ odd \textbf{and} previous sample reusable}
    \State reuse Sobol prefix; re-cluster with $\theta$
           \Comment{0 evaluations}
  \Else
    \State evaluate Sobol design $P$, $|P|=M$ \Comment{$t\gets t+M$}
  \EndIf
  \State build NB tree~\eqref{eq:nbtree}; select $\phi$ by
         staircase; basins $B_1{\le}\cdots{\le}B_K$
  \State $E\gets\emptyset$ \Comment{excluded basins, per cycle}
  \For{$B$ in ascending size, while $t<T_{\mathrm{main}}$}
    \State $x_0\gets$ best point of $B$;\
           \textbf{if} vote$(x_0)\in E$ \textbf{skip}
    \State run CMA-ES$(x_0,\sigma_0~\eqref{eq:sigma0},
           \lambda~\eqref{eq:popsize})$ until stop
           (Sec.~\ref{sec:phase1})
    \State $\beta\gets$ vote(best point of run);\
           exclude $\beta$ into $E$ on its 2nd convergence
  \EndFor
  \State $c\gets c+1$
\EndWhile
\If{$T-t\ge 10D$}
  \State run CMA-ES from $x^{\mathrm{best}}$ with all tolerance
         stops disabled until $t=T$ \Comment{refinement}
\EndIf
\State \Return $x^{\mathrm{best}}$
\end{algorithmic}
\end{algorithm}

\begin{table}[t]
\centering
\caption{The two configurations (Optuna \cite{Akiba2019} best trials, tuned on CEC2017
  $D = 10$ at the official budget; exact values in the repository).
  $\delta_{\mathrm{ref}}$ is anchored at $D = 10$ and rescaled by
  $\sqrt{10/D}$, Eq.~\eqref{eq:sigma0}; all other parameters are
  universal across dimensions and budgets.}
\label{tab:configs}
\small
\begin{tabular}{lccl}
\toprule
Parameter & $\theta_C$ & $\theta_B$ & Role \\
\midrule
$M$ (Phase-0 sample)        & 4096   & 4096   & Sobol design size \\
$n_b$ (target basins)       & 25     & 5      & staircase target \\
$s_{\min}$ (useful basin)   & 7      & 155    & min.\ basin size \\
$b$ (Rule-2 ratio)          & 2.498  & 3.029  & hub cut, Eq.~\eqref{eq:nbc-rule2} \\
$n_{\min}$ (Rule-2 in-degree)      & 3      & 3      & hub cut \\
$\varepsilon$ (elite frac.) & 0.411  & 0.065  & elite centre, Eq.~\eqref{eq:sigma0} \\
$\delta_{\mathrm{ref}}$     & 5.989  & 1.661  & $\sigma_0$ divisor \\
$\rho$ (popsize frac.)      & 0.129  & 0.434  & Eq.~\eqref{eq:popsize} \\
$\lambda_{\max}$            & 12     & 236    & popsize cap \\
$\tau_f$ / $\tau_x$         & 4 / 8  & 4 / 8  & CMA tolerances \\
$s_{\mathrm{tol}}$          & 13.82  & 13.82  & population-range stop \\
$r_C$ (refine reservation)    & 0.04   & 0.04   & Sec.~\ref{sec:refinement} \\
$k$-NN vote                 & 5      & 5      & membership \\
\bottomrule
\end{tabular}
\end{table}

\section{Experiments and results}\label{sec:experiments}

\subsection{Protocol}\label{sec:protocol}

We evaluate all algorithms on four benchmark suites from the IEEE
Congress on Evolutionary Computation (CEC), observing strictly their
official maximum evaluation budgets.  Within the design envelope of the
proposed method ($D\le 20$) the experimental cells (budgets) are: CEC2014
$D = 10$ ($10^{5}$); CEC2017 $D = 10$ ($10^{5}$); CEC2020 $D = 5$
($5\times10^{4}$), $D = 10$ ($10^{6}$), $D = 15$ ($3\times10^{6}$), and
$D = 20$ ($10^{7}$); and CEC2022 $D = 10$ ($2\times10^{5}$) and $D = 20$
($10^{6}$)---eight cells in all.  Two further cells outside the
envelope, CEC2014 and CEC2017 at $D = 30$ ($3\times10^{5}$), are
examined separately to probe the scalability boundary.  Following the
suite authors' recommendation, the deprecated function $f_2$ of CEC2017
is excluded for all algorithms, leaving $123$ functions across the eight
envelope cells.

The BIPOP-CMA-ES baseline is evaluated using the \texttt{pycma} reference 
implementation \cite{Hansen2019pycma} with its native restart logic and the standard initial step-size 
$\sigma_0=(u-l)/4$. The five DE baselines, all descendants of the (L-)SHADE family~\cite{Tanabe2014}, are executed via their respective \texttt{minionpy} 
C++ reference implementations \cite{Minion2025}, configured with the author-recommended default parameters.

MSC-CMA-ES operates under the alternating C/B schedule described in Section~\ref{sec:cycles}, 
utilizing the two fixed configurations detailed in Table~\ref{tab:configs}. To assess the 
generalization capability and structural robustness of the approach, no per-suite or 
per-dimension hyperparameter tuning was conducted. Both configurations were tuned exactly once on the CEC2017 $D=10$ cell at its official budget; consequently, all other parameters are universal across dimensions and budgets. The class~B parameters were tuned on the unimodal, simple-multimodal, and hybrid functions ($f_1$--$f_{20}$), and the class~C parameters on the
composition functions ($f_{21}$--$f_{30}$).

Each (algorithm, function, cell) configuration is evaluated over 51 independent runs using 
distinct random seeds from $\{0, \dots, 50\}$. Final optimization errors satisfying 
$|f(\mathbf{x}) - f^*| \le 10^{-8}$ are floored to zero, matching the global target threshold 
floor. For each individual function, the Supplementary Material provides the mean, median, and best final errors, alongside the fixed-budget target
coverage (FBTC): with $51$ log-uniformly spaced targets
$\tau\in[10^{-8},10^{2}]$, FBTC is the fraction of (target, run) pairs
whose final error satisfies $e\le\tau$.  It lies in $[0,1]$ per function
and equals the terminal cross-section of the COCO-style target-hit
profile; the result tables report its sum over the functions of each
class or dimension.
Statistical significance is evaluated at 
a per-function level via the Wilcoxon signed-rank test ($\alpha=0.05$), applying the Benjamini--Hochberg false discovery rate (FDR) correction 
within each (suite, dimension) cell, separately for each baseline-MSC-CMA-ES comparison, over the functions of that cell.

Environment: Python 3.13.5 (anaconda3 env intelpython), NumPy 2.3.1, SciPy 1.15.3, pycma 4.4.2, minionpy 1.5.0. Hardware: Intel Xeon Platinum 8160 @ 2.10 GHz, 192 threads, 251 GiB RAM.

\subsection{Per-dimension results at official budgets}\label{sec:per-dim}

Table~\ref{tab:dim-summary} reports the summed mean, median, and
fixed-budget target coverage (FBTC) per dimension at the official CEC
budgets, for $D\in\{5,10,15,20,30\}$.  At the official budget,
MSC-CMA-ES attains the lowest summed mean and median error at every
dimension up to and including $D = 20$.

At $D = 30$ (budget $10^{4}\,D$) the lowest summed mean and median pass
to the DE baselines, with MSC-CMA-ES mid-pack.  This is the boundary of
the design envelope of Section~\ref{sec:protocol}.  Preuss notes that
niching is not a universal cure: once a landscape is multimodal enough
that not enough basins of attraction can be identified from a sample of
affordable size, coordinated niching gains no clear advantage over
uncoordinated restarts~\cite{Preuss2010,Preuss2012}.
A similar phenomenon is reported in~\cite{ARRDE2023}: restarts become less effective as dimension grows, and at large $D$ ARRDE shifts weight to its underlying LSHADE/jSO search dynamics.
A method built on NBC basin
identification loses its low-dimensional advantage in this regime,
which is consistent with the observed $D = 30$ ordering; the present
data do not by themselves isolate the cause.  Two DE baselines (j2020,
NL-SHADE-RSP) diverge on a subset of the $D = 30$ functions, raising
their summed means more than an order of magnitude above the rest of
the field.

On fixed-budget target coverage the picture is complementary: a DE
baseline attains the highest FBTC at every dimension---NL-SHADE-RSP at
$D = 5$ and $D = 10$, j2020 at $D = 15$, ARRDE at $D = 20$, and L-SRTDE
at $D = 30$---while neither CMA variant leads coverage at any
dimension.  The summed errors live on very different scales across $D$
(tens at $D = 5$, thousands at $D = 30$), so the per-dimension rows are
the comparable unit; a cross-dimension total would be dominated by the
$D = 30$ divergences and is not reported.

\begin{table}[htbp]\centering
\caption{Per-dimension totals at the official CEC budgets. Lowest error / highest FBTC per row in bold.}
\label{tab:dim-summary}\small
\begin{tabular}{ll||rr|rrrrr}
\toprule
$D$ & Metric & MSC-CMA-ES & BIPOP-CMA-ES & ARRDE & j2020 & jSO & L-SRTDE & NL-SHADE-RSP \\
\midrule
\multirow{3}{*}{$\substack{\text{\small $D = 5$}\\ \text{\small (10)}}$}
 & mean   & \textbf{74.37} & 391.23 & 159.08 & 189.14 & 452.12 & 482.34 & 227.22 \\
 & median & \textbf{4.68} & 456.65 & 105.79 & 105.97 & 452.87 & 453.24 & 301.77 \\
 & FBTC   & 5.53 & 4.72 & 6.73 & 6.15 & 5.84 & 5.07 & \textbf{7.83} \\
\midrule
\multirow{3}{*}{$\substack{\text{\small $D = 10$}\\ \text{\small (81)}}$}
 & mean   & \textbf{4400.09} & 6266.20 & 4699.62 & 5408.53 & 5948.12 & 38389.88 & 5219.56 \\
 & median & \textbf{4715.55} & 6322.05 & 4933.86 & 5336.90 & 6064.13 & 6388.68 & 5300.49 \\
 & FBTC   & 27.40 & 26.09 & 32.42 & 28.27 & 32.00 & 29.88 & \textbf{33.54} \\
\midrule
\multirow{3}{*}{$\substack{\text{\small $D = 15$}\\ \text{\small (10)}}$}
 & mean   & \textbf{280.88} & 638.44 & 407.84 & 557.14 & 906.86 & 938.01 & 563.23 \\
 & median & \textbf{215.06} & 624.94 & 354.44 & 526.28 & 883.77 & 907.83 & 531.48 \\
 & FBTC   & 3.09 & 3.33 & 2.47 & \textbf{3.52} & 2.10 & 2.12 & 3.46 \\
\midrule
\multirow{3}{*}{$\substack{\text{\small $D = 20$}\\ \text{\small (22)}}$}
 & mean   & \textbf{1023.37} & 1366.02 & 1139.29 & 1747.65 & 1827.39 & 1811.99 & 1720.27 \\
 & median & \textbf{1041.84} & 1246.21 & 1141.68 & 1775.53 & 1825.42 & 1798.42 & 1635.75 \\
 & FBTC   & 7.15 & 7.88 & \textbf{8.02} & 6.52 & 5.87 & 6.44 & 7.15 \\
\midrule
\multirow{3}{*}{$\substack{\text{\small $D = 30$}\\ \text{\small (59)}}$}
 & mean   & 15272.54 & 16831.29 & 13702.79 & 467983.24 & 11779.09 & \textbf{8630.72} & 438534.58 \\
 & median & 14933.12 & 16162.84 & 13226.23 & 358024.96 & 11761.69 & \textbf{8395.07} & 277179.35 \\
 & FBTC   & 8.45 & 12.16 & 14.11 & 9.17 & 14.50 & \textbf{15.83} & 8.55 \\
\bottomrule
\end{tabular}
\end{table}

\subsection{Per-class results at official budgets}\label{sec:per-type}

Table~\ref{tab:class-summary} reports the summed mean, median, best, and
FBTC over the same official-budget cells, grouped by function class
(unimodal and simple-multimodal, hybrid, composition); CEC2017 $f_2$ is excluded, leaving $123$
functions. The three classes give three different orderings.

On composition functions MSC-CMA-ES attains the best value on all four
aggregates: the lowest summed mean, median, and best error, and the
highest FBTC ($9.69$, compared to $3.56$ for BIPOP-CMA-ES and $8.45$ for the
strongest DE baseline, NL-SHADE-RSP).  On unimodal and simple-multimodal functions it attains the
lowest summed median ($138.79$) but the lowest FBTC of the field
($25.46$, compared to $27.04$--$30.63$ for the others); its summed mean
($221.10$) is within $0.12$ of the lowest (L-SRTDE, $220.98$).  On hybrid
functions the leading DE algorithms are several-fold to an order of
magnitude ahead on mean and median ($39.89$/$33.26$ for jSO/L-SRTDE
compared to $265.11$/$294.35$ for MSC-CMA-ES), and every DE baseline
outscores both CMA variants on FBTC; the two CMA variants sit together
at the back, with the exception that one DE baseline, NL-SHADE-RSP, is
worse than both on mean and median.

These orderings are consistent with the per-dimension picture of
Section~\ref{sec:per-dim}.  Structure-aware restarts are designed to act on the composition class.
The composition-class results are consistent with this design intent.

Table~\ref{tab:class-summary} is restricted to the cells with $D\le 20$,
the design envelope of Section~\ref{sec:protocol}: NBC basin identification
is reliable only while basins remain recognizable from a sample of
affordable size~\cite{Preuss2010,Preuss2012}.  The $D = 30$ cells are
reported per dimension in Table~\ref{tab:dim-summary}.

\begin{table}[htbp]
\centering
\caption{Per-class totals over the eight official-budget cells with
  $D\in\{5,10,15,20\}$. Lowest error or highest FBTC value per row in bold.}
\label{tab:class-summary}
\small
\begin{tabular}{ll||rr|rrrrr}
\toprule
Class & Metric & MSC-CMA-ES & BIPOP-CMA-ES & ARRDE & j2020 & jSO & L-SRTDE & NL-SHADE-RSP \\
\midrule
\multirow{4}{*}{\shortstack[l]{Unimodal +\\ simple-mult.}}
 & mean   & 221.10 & 313.02 & 312.15 & 301.22 & 345.24 & \textbf{220.98} & 406.37 \\
 & median & \textbf{138.79} & 231.64 & 249.80 & 251.47 & 337.87 & 194.11 & 437.91 \\
 & best   & 36.71 & \textbf{15.13} & 55.71 & 25.87 & 110.91 & 103.89 & 93.98 \\
 & FBTC   & 25.46 & 29.85 & 29.64 & 27.04 & 28.35 & 28.96 & \textbf{30.63} \\
\midrule
\multirow{4}{*}{Hybrid}
 & mean   & 265.11 & 357.75 & 63.74 & 254.48 & \textbf{39.89} & 58.17 & 524.30 \\
 & median & 294.35 & 315.70 & 36.06 & 180.99 & 39.38 & \textbf{33.26} & 377.36 \\
 & best   & 12.09 & 6.54 & \textbf{1.33} & 27.70 & 3.22 & 2.83 & 9.28 \\
 & FBTC   & 8.01 & 8.60 & 14.47 & 12.25 & \textbf{15.20} & 12.40 & 12.90 \\
\midrule
\multirow{4}{*}{Composition}
 & mean   & \textbf{5292.49} & 7991.12 & 6029.94 & 7346.77 & 8749.36 & 41325.87 & 6799.60 \\
 & median & \textbf{5543.97} & 8102.51 & 6249.91 & 7312.22 & 8848.94 & 9320.80 & 6954.21 \\
 & best   & \textbf{3008.56} & 4668.57 & 3211.60 & 4795.82 & 7962.90 & 7596.05 & 4037.85 \\
 & FBTC   & \textbf{9.69} & 3.56 & 5.52 & 5.17 & 2.26 & 2.15 & 8.45 \\
\bottomrule
\end{tabular}
\end{table}

\subsection{Per-function comparison}\label{sec:per-function}

At the per-function level the class pattern of Table~\ref{tab:class-summary} persists: MSC-CMA-ES is significantly stronger on composition functions, while the DE baselines are significantly better on the majority of hybrid functions. The complete per-function tables---mean, median, best, and FBTC with per-function Wilcoxon signed-rank tests (Benjamini--Hochberg FDR, $\alpha = 0.05$) against MSC-CMA-ES---are given in the Supplementary Material for every (suite, dimension) cell.

\subsection{Budget scaling}\label{sec:budget}

We evaluate the algorithms across multiple budgets to analyze how the final accuracy depends on the  maximum number of function evaluations.

The raw fixed-budget target coverage $\mathrm{FBTC}(b)$ defined in Section~\ref{sec:protocol} is not necessarily a non-decreasing function of the budget $b$. Unlike the runtime perspective in COCO, which is monotonically non-decreasing by construction because it records the first hitting time \cite{HansenAnytime2022,HansenCOCO2021}, a budget-indexed quality measure carries no such guarantee. 

To resolve this, we replace $\mathrm{FBTC}(b)$ with its monotone envelope, defining the coverage at budget $b$ as $\max_{b' \le b} \mathrm{FBTC}(b')$. This provides a fixed-budget analogue to the best-so-far convergence curves used in anytime performance assessment \cite{HansenAnytime2022,Wang2022iohanalyzer}.

\begin{definition}\label{def:fbtc-mono}
Let $b_1 < b_2 < \dots < b_K$ be the budgets available for a given
(suite, dimension, function) cell, and let $\mathrm{FBTC}(b_k)\in[0,1]$ be the
raw coverage of Section~\ref{sec:protocol} at budget $b_k$. The monotone
envelope is the running maximum
$$
  \widehat{\mathrm{FBTC}}(b_k)=\max_{1\le j\le k}\mathrm{FBTC}(b_j), \qquad k=1,\dots,K.
$$
\end{definition}
For simplicity, in the remainder of the paper we use $\mathrm{FBTC}$ to refer to the monotone envelope $\widehat{\mathrm{FBTC}}$.

\begin{figure*}[t]\centering
  \begin{subfigure}{0.24\linewidth}\centering
    \includegraphics[width=\linewidth]{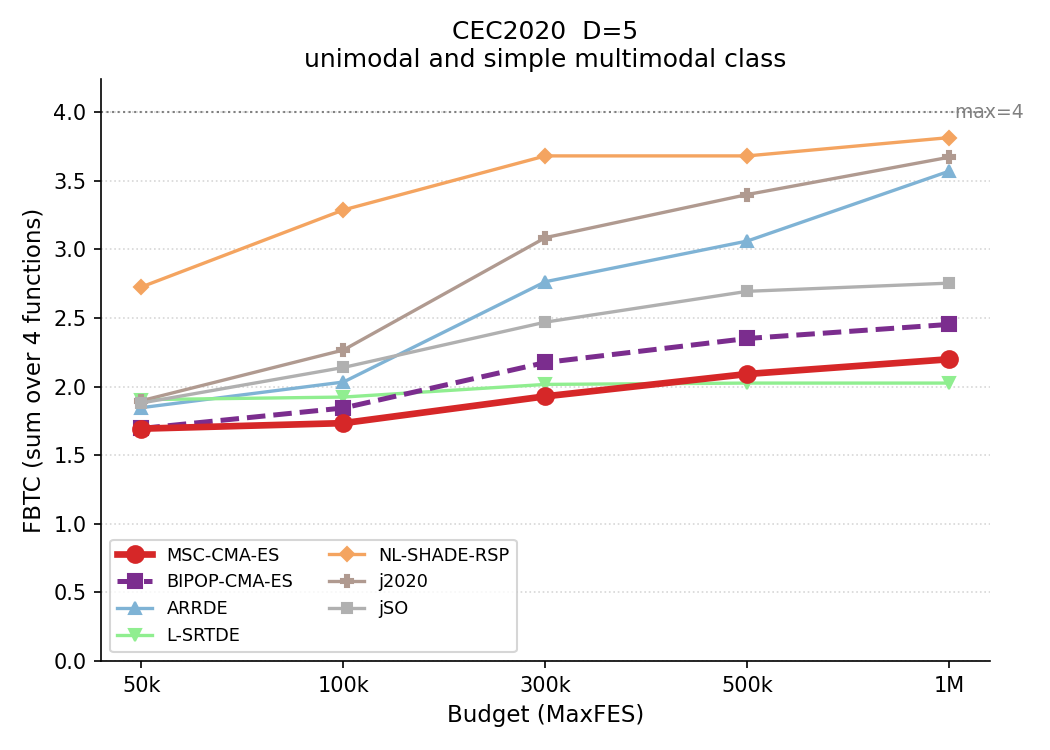}
    \caption{$f_1$--$f_4$.}\label{fig:fbtc-bh-d5-basic}
  \end{subfigure}\hfill
  \begin{subfigure}{0.24\linewidth}\centering
    \includegraphics[width=\linewidth]{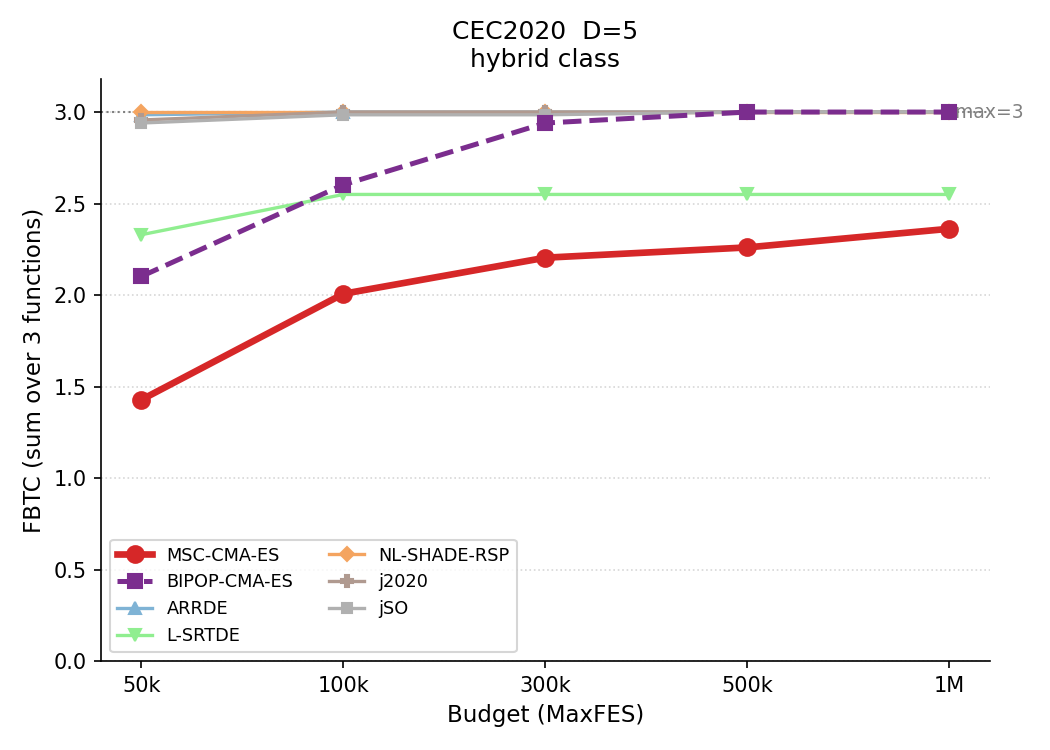}
    \caption{$f_5$--$f_7$.}\label{fig:fbtc-bh-d5-hybrid}
  \end{subfigure}
  \begin{subfigure}{0.24\linewidth}\centering
    \includegraphics[width=\linewidth]{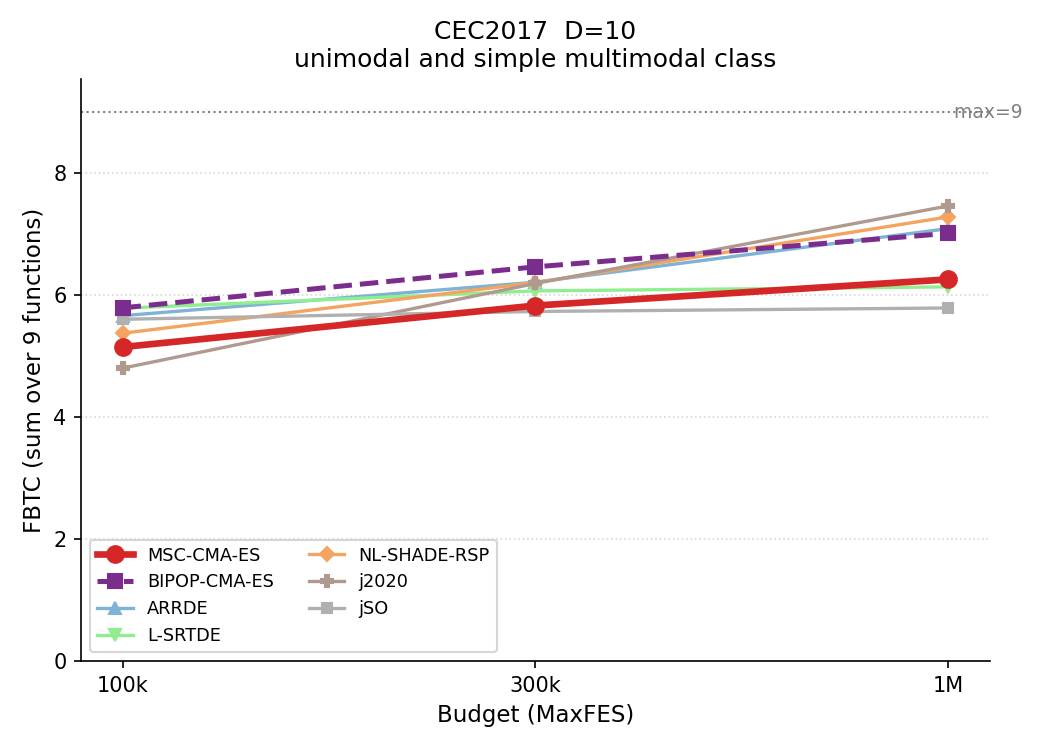}
    \caption{$f_1$--$f_{10}$.}\label{fig:fbtc-bh-d10-basic}
  \end{subfigure}\hfill
  \begin{subfigure}{0.24\linewidth}\centering
    \includegraphics[width=\linewidth]{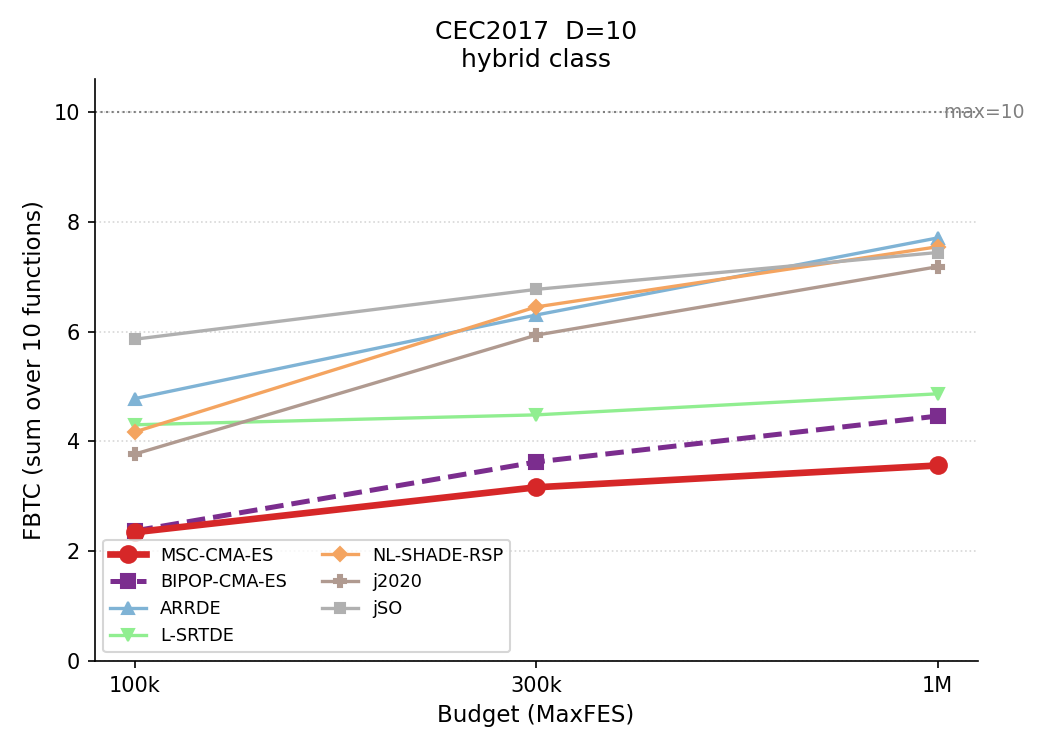}
    \caption{$f_{11}$--$f_{20}$.}\label{fig:fbtc-bh-d10-hybrid}
  \end{subfigure}
  \caption{FBTC on the unimodal, simple-multimodal and hybrid classes of CEC2020 $D = 5$ -- (a), (b) and CEC2017 $D = 10$ -- (c), (d), summed over the functions.
  }\label{fig:fbtc-bh}
\end{figure*}

On the unimodal, simple-multimodal and hybrid classes of CEC2020 $D=5$ and CEC2017 $D=10$ (Fig.~\ref{fig:fbtc-bh}), MSC-CMA-ES yields one of the three lowest $\mathrm{FBTC}$ values at each budget. Coverage increases with the budget across all four panels. For instance, on the CEC2017 $D=10$ hybrid class, MSC-CMA-ES improves from 2.35 to 3.56 (out of a maximum of 10) between $10^5$ and $10^6$ evaluations. Its growth rate is $+1.22$ per decade, compared to $+2.9$ to $+3.4$ for the leading DE baselines, maintaining a constant per-class ranking over the evaluated range. An exception occurs on the CEC2020 $D=5$ hybrid class (Fig.~\ref{fig:fbtc-bh-d5-hybrid}): the DE baselines reach maximum coverage at the lowest budget, while the MSC-CMA-ES growth rate ($+0.85$) closely matches that of BIPOP-CMA-ES ($+0.90$).

\begin{figure*}[htpb]
    \centering
    \begin{subfigure}{0.32\textwidth}
        \includegraphics[width=\linewidth]{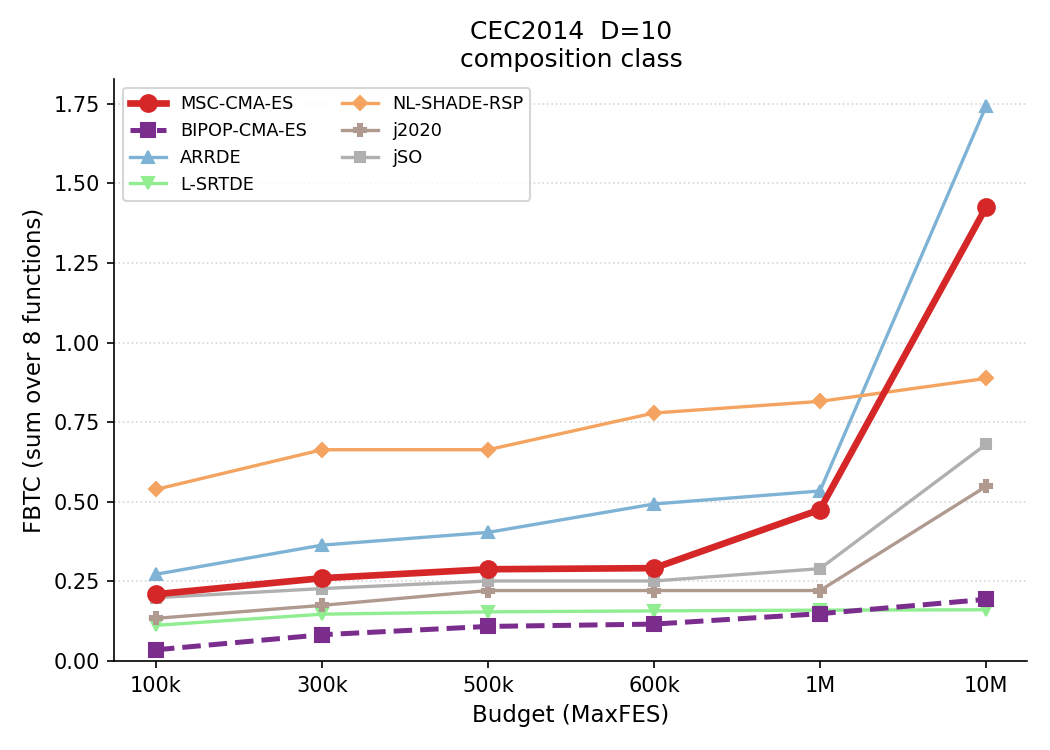}
        \caption{CEC2014}
        \label{fig:nfl_cec2014}
    \end{subfigure}\hfill
    \begin{subfigure}{0.32\textwidth}
        \includegraphics[width=\linewidth]{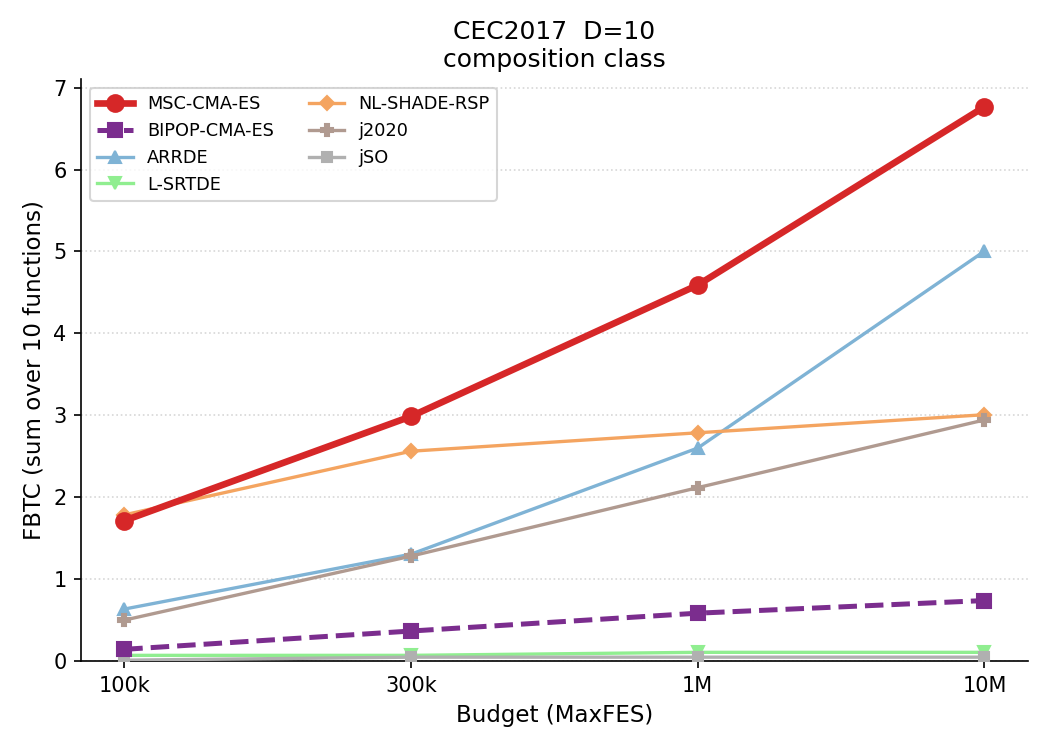}
        \caption{CEC2017}
        \label{fig:nfl_cec2017}
    \end{subfigure}\hfill
    \begin{subfigure}{0.32\textwidth}
        \includegraphics[width=\linewidth]{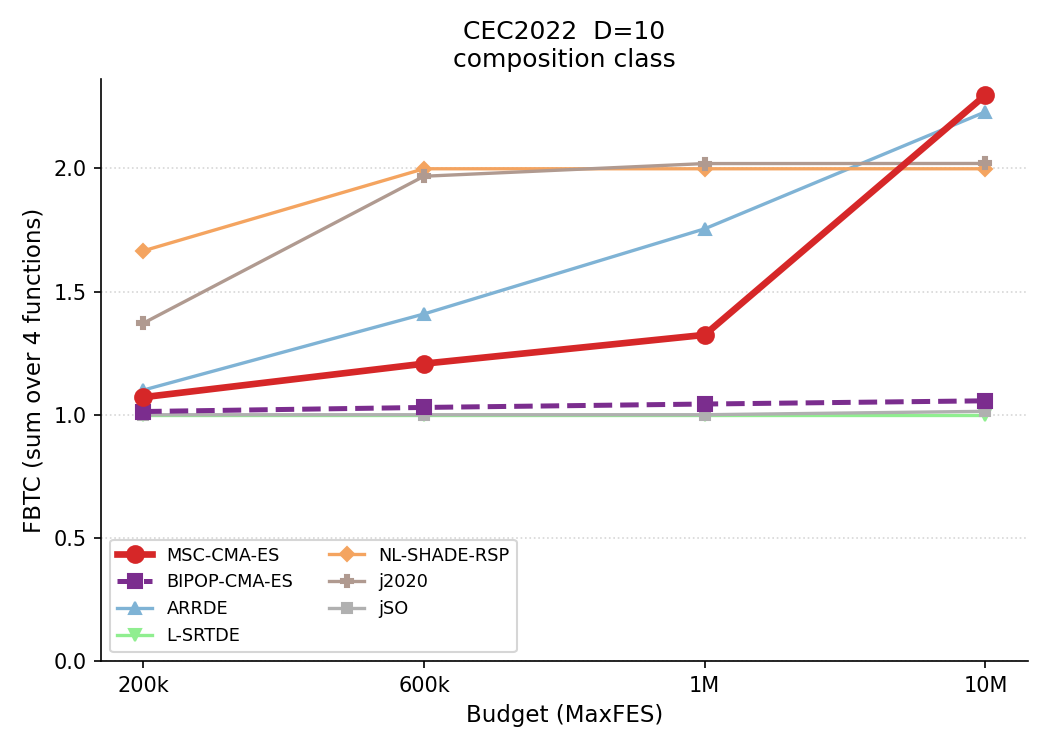}
        \caption{CEC2022}
        \label{fig:nfl_cec2022}
    \end{subfigure}
    
    \caption{Composition functions at $D=10$ across the CEC2014, CEC2017, and CEC2022 suites.}
    \label{fig:nflt-comparison}
\end{figure*}

On composition functions at $D=10$ across the three suites, MSC-CMA-ES attains the highest $\mathrm{FBTC}$ on the CEC2017 set, while NL-SHADE-RSP and j2020 attain higher $\mathrm{FBTC}$ on the CEC2014 and CEC2022 suites.

\begin{figure*}[t]\centering
  \begin{subfigure}{0.24\linewidth}\centering
    \includegraphics[width=\linewidth]{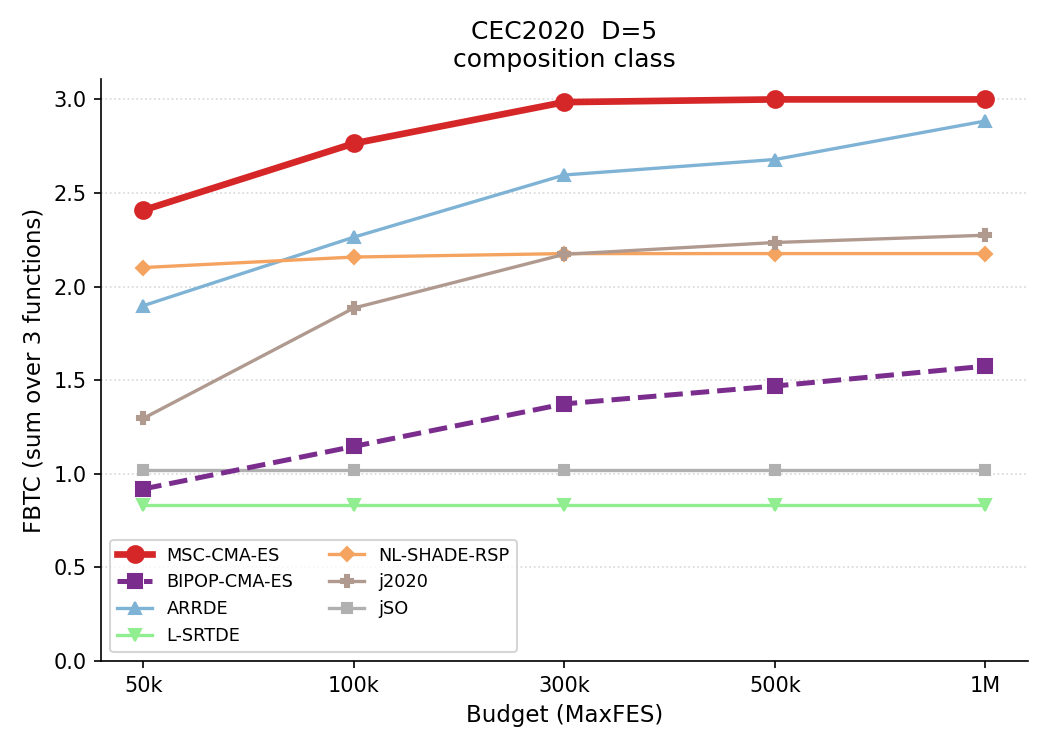}
    \caption{$D = 5$.}\label{fig:fbtc-comp-d5}
  \end{subfigure}\hfill
  \begin{subfigure}{0.24\linewidth}\centering
    \includegraphics[width=\linewidth]{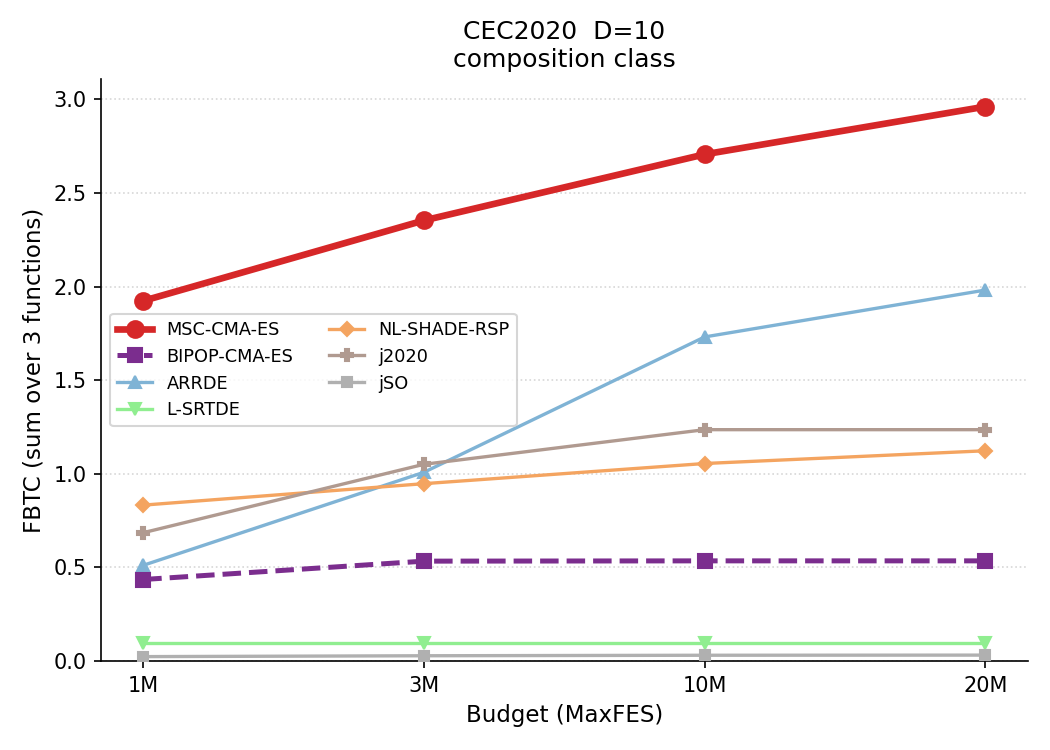}
    \caption{$D = 10$.}\label{fig:fbtc-comp-d10}
  \end{subfigure}
  \begin{subfigure}{0.24\linewidth}\centering
    \includegraphics[width=\linewidth]{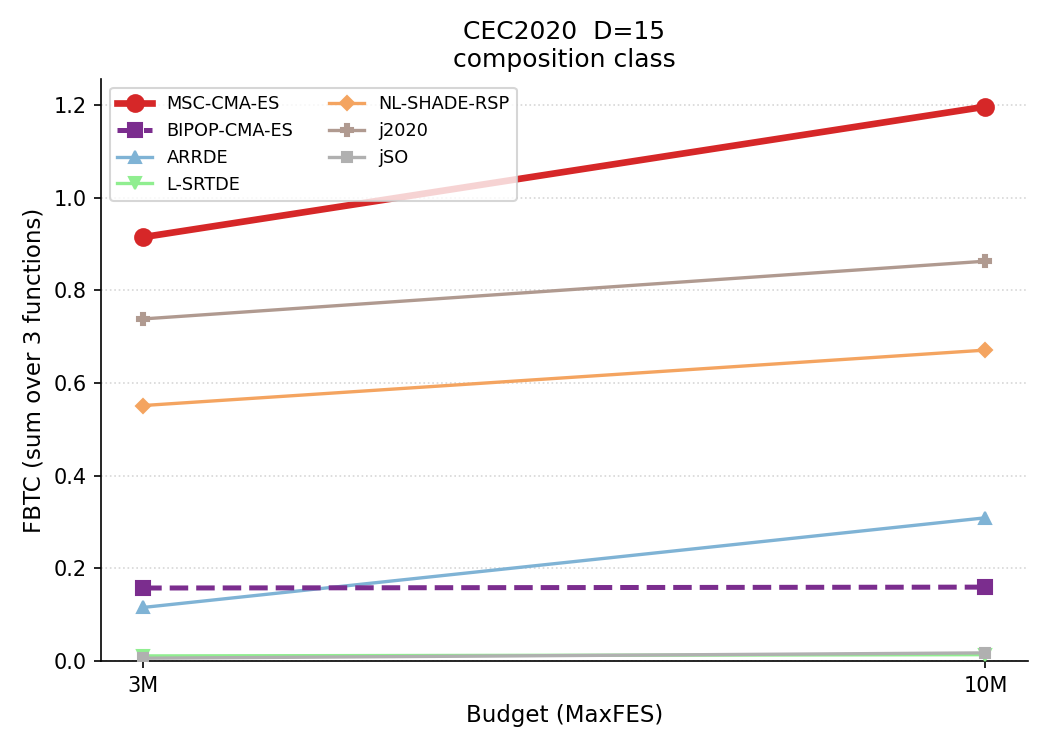}
    \caption{$D = 15$.}\label{fig:fbtc-comp-d15}
  \end{subfigure}\hfill
  \begin{subfigure}{0.24\linewidth}\centering
    \includegraphics[width=\linewidth]{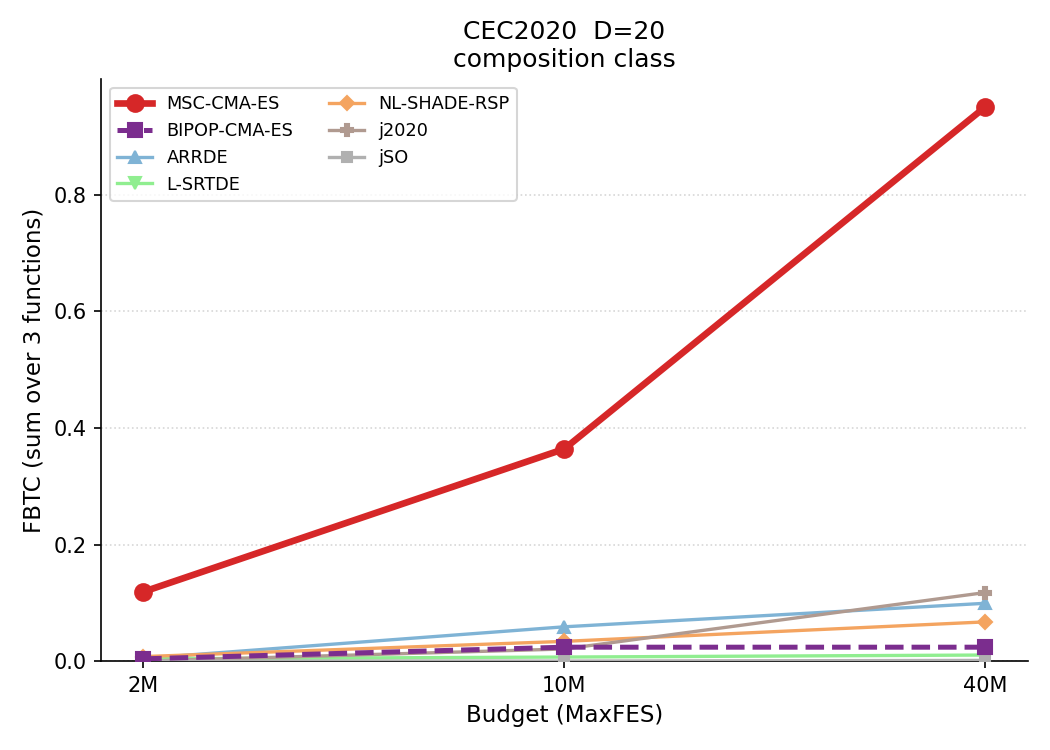}
    \caption{$D = 20$.}\label{fig:fbtc-comp-d20}
  \end{subfigure}
  \caption{FBTC on the composition class ($f_8$--$f_{10}$,
  class maximum $3$) of CEC2020 at $D\in\{5,10,15,20\}$.}\label{fig:fbtc-comp}
\end{figure*}
Across all dimensionalities of the CEC2020 composition class ($D \in \{5, 10, 15, 20\}$), MSC-CMA-ES attains the highest $\mathrm{FBTC}$ at every $D$.
At $D=5$ it reaches the maximum coverage of 3.0 by approximately $5 \times 10^5$ evaluations.

\section{Discussion}\label{sec:discussion}
We now examine what the aggregate rankings and the budget-scaling behaviour
imply about where and why the structure-aware restarts help.

On the composition class the three magnitude measures (worst-, median-, and
best-SUM) order the algorithms almost identically, whereas FBTC reorders the
field. On CEC2014 $D = 10$ (Fig.~\ref{fig:rank-2014d10}) FBTC lifts
NL-SHADE-RSP from third on median-SUM to first and drops MSC-CMA-ES from first
to third, and on CEC2022 $D = 20$ (Fig.~\ref{fig:rank-2022d20}) it moves
BIPOP-CMA-ES from second to fourth.

At the official $3\times10^{5}$ budget MSC-CMA-ES is mid-pack on the CEC2017
$D = 30$ composition class (fourth on worst-SUM, third on median- and
mean-SUM, second on best-SUM), but at $10^{6}$ evaluations it attains the
lowest worst-, mean-, and best-SUM and nearly ties ARRDE on median-SUM.
The reversal suggests that the D=30 loss observed at the official budget is at least partly due to budget limitation, rather than purely intrinsic:  in higher dimension a fixed Sobol sample yields a sparser basin model, so more budget is needed before the information gathered in Phase~0 translates into better restarts.

\begin{figure*}[htpb]
    \centering
    \begin{subfigure}{0.48\textwidth}
        \includegraphics[width=\linewidth]{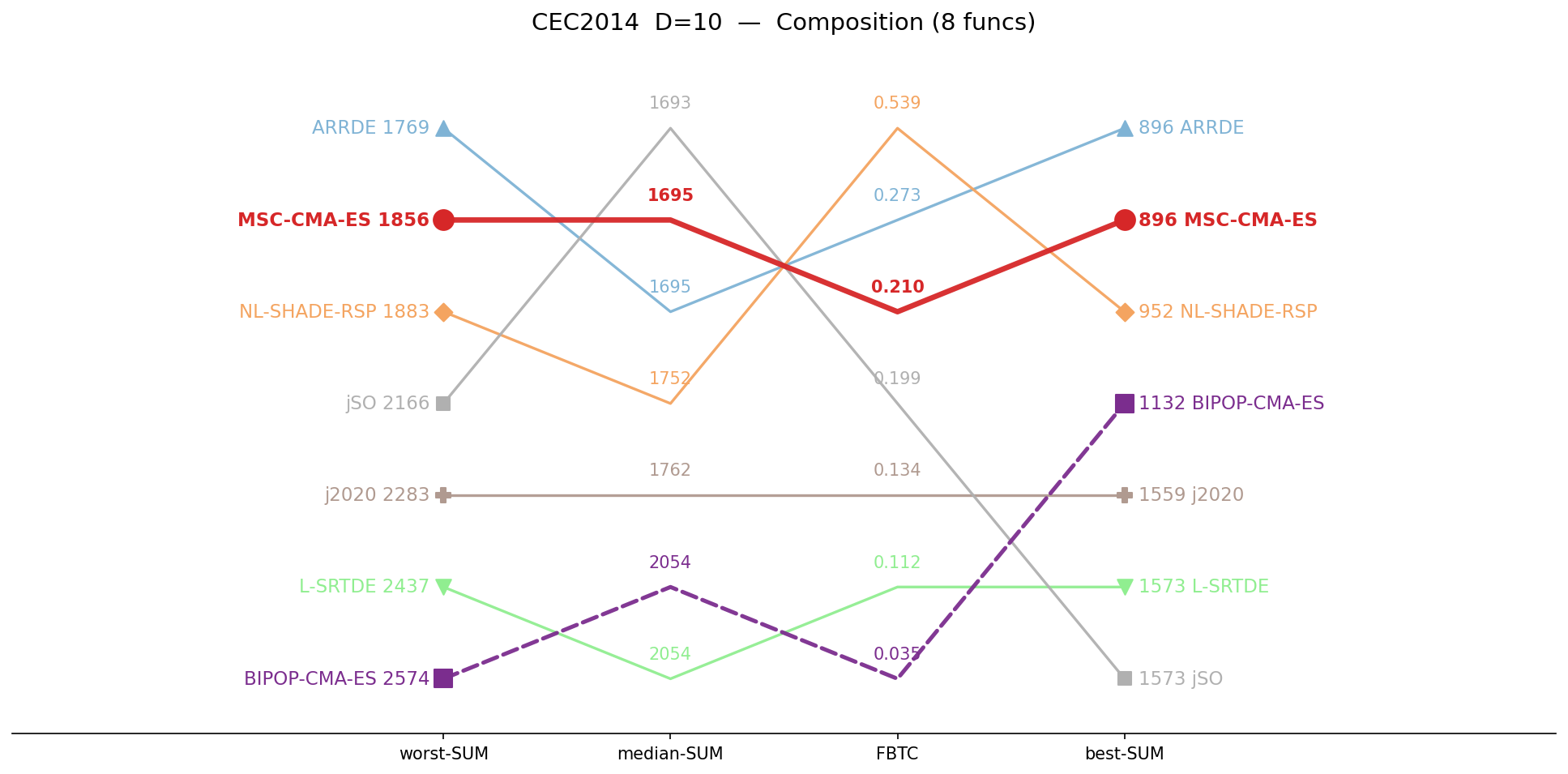}
        \caption{CEC2014 ($D=10$)}
        \label{fig:rank-2014d10}
    \end{subfigure}\hfill
    \begin{subfigure}{0.48\textwidth}
        \includegraphics[width=\linewidth]{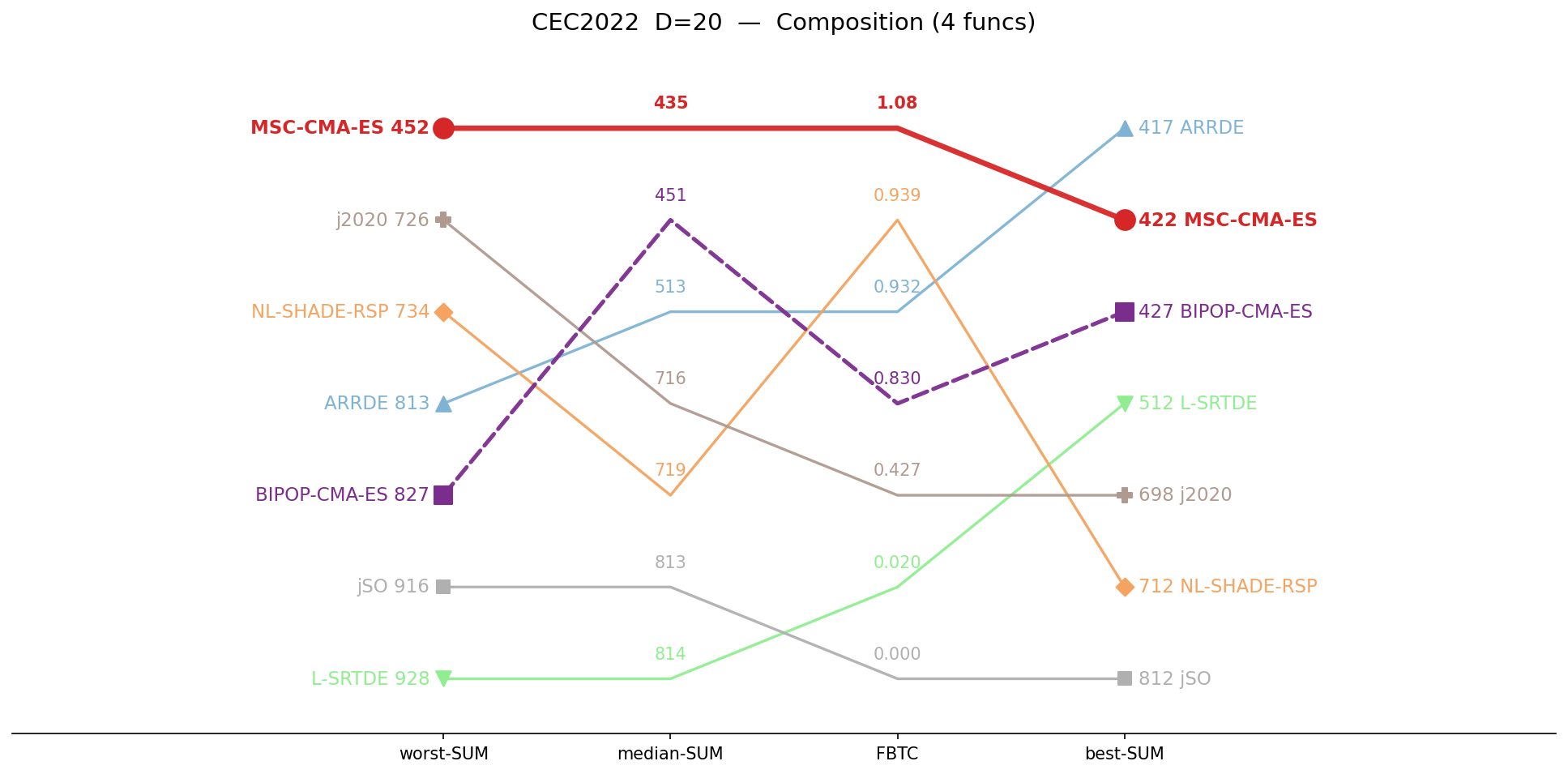}
        \caption{CEC2022 ($D=20$)}
        \label{fig:rank-2022d20}
    \end{subfigure}
    
    \caption{Algorithm comparison on the composition classes across four aggregate measures.}
    \label{fig:rank-comparison-all}
\end{figure*}

\begin{figure*}[t]\centering
  \begin{subfigure}{0.49\textwidth}\centering
    \includegraphics[width=\linewidth]{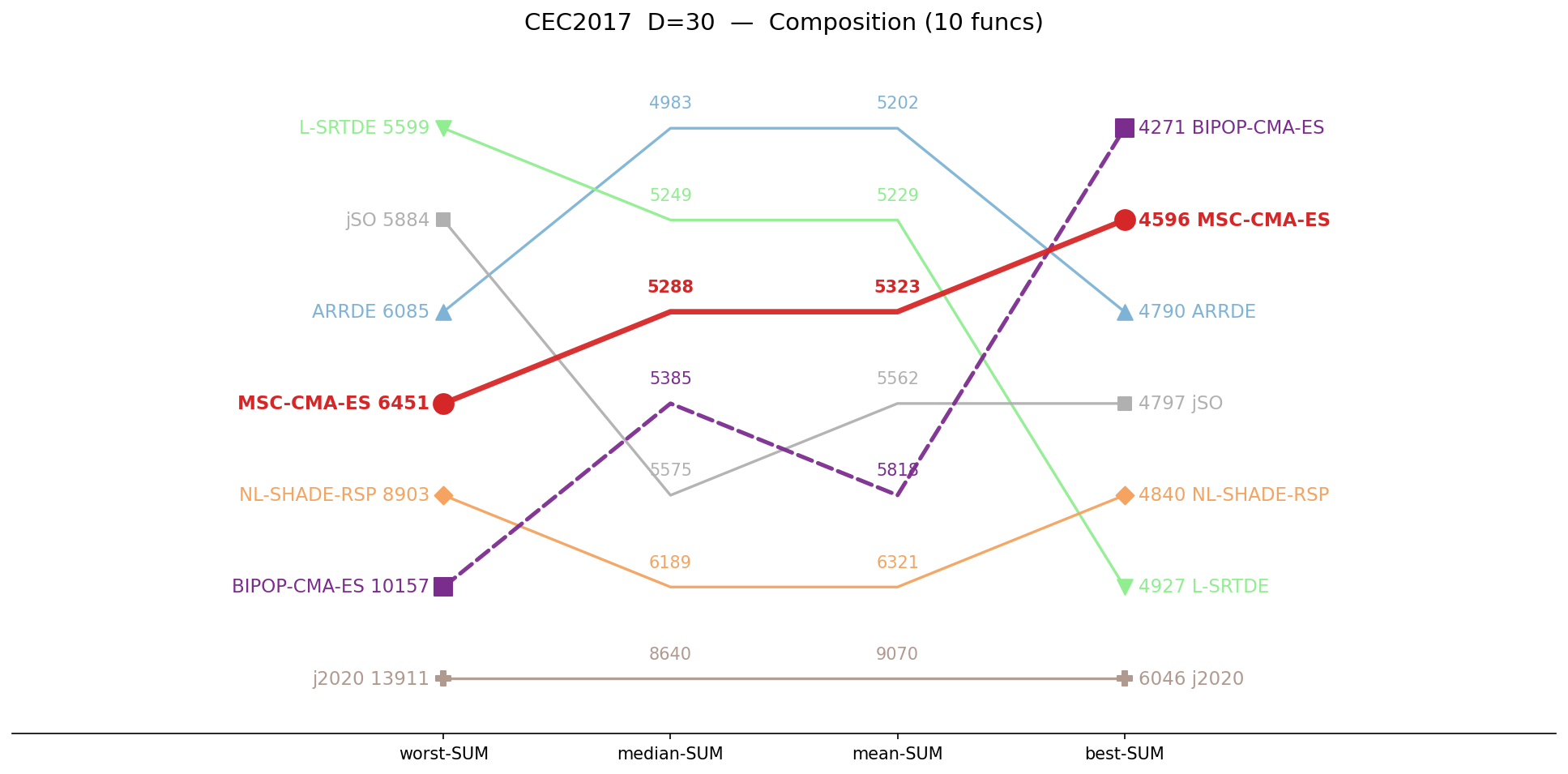}
    \caption{Official budget, $3\times 10^5$ evaluations.}
    \label{fig:d30-comp-official}
  \end{subfigure}\hfill
  \begin{subfigure}{0.49\textwidth}\centering
    \includegraphics[width=\linewidth]{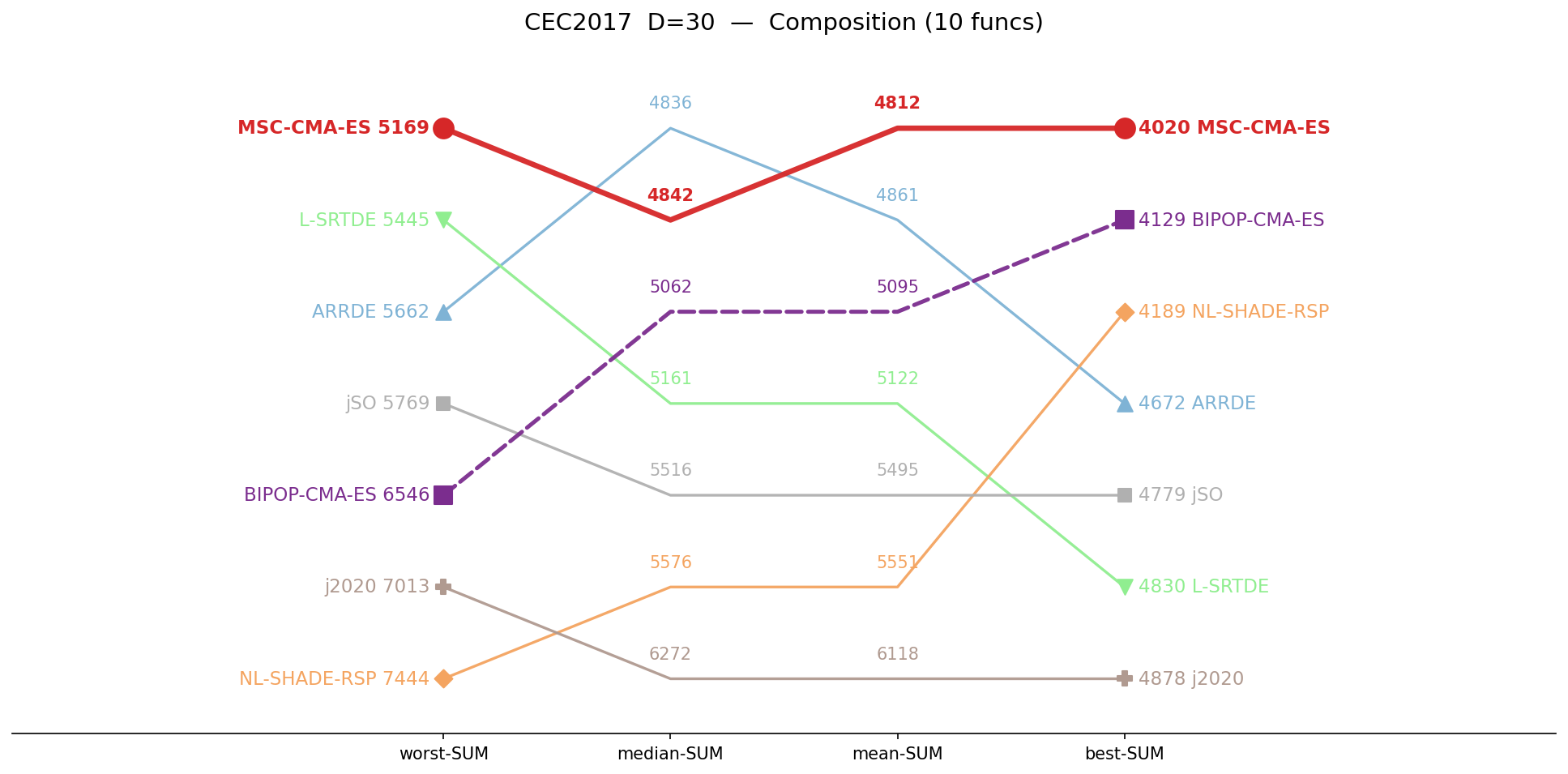}
    \caption{Extended budget, $10^6$ evaluations.}
    \label{fig:d30-comp-1m}
  \end{subfigure}

  \caption{CEC2017 $D = 30$ composition-class stress test under two
  evaluation budgets.}
  \label{fig:d30-comp-budget}
\end{figure*}

As dimensionality increases, the gap between MSC-CMA-ES and the baselines
grows: no baseline exceeds an $\mathrm{FBTC}$ of 1.98 at $D = 10$, 1.5 at
$D = 15$, or 1.2 at $D = 20$, whereas MSC-CMA-ES keeps increasing its coverage
monotonically over the evaluated budgets. On the $D = 20$ composition class it
rises from $\approx1.1$ at $2\times10^{6}$ evaluations to 2.6 (of a maximum 3)
at $4\times10^{7}$ (Fig. \ref{fig:fbtc-comp-d20}).

\begin{table}[t]
\centering
\caption{Per-cycle-class budget usage and refinement contribution; B/C cyc: median
count of B- and C-class cycles per run. seed\,B/C: share of runs whose
refinement seed---the last cycle to lower the global best---is a B resp.\ C
cycle. Ref: share of the budget spent in refinement. Improv:
median $\log_{10}(f_{\mathrm{pre}}/f_{\mathrm{final}})$, the order-of-magnitude
error reduction contributed by refinement.}
\label{tab:cycle-anatomy}
\begin{tabular}{@{}llrrrrrr@{}}
\toprule
Class & Budget & B\,cyc & C\,cyc & seed\,B & seed\,C & Ref & Improv \\
\midrule
\multirow{3}{*}{Hybrid}
 & $10^5$          & 1  & 1  & 89\% & 11\% & 4.0\% & 0.80 \\
 & $3{\times}10^5$ & 3  & 3  & 98\% & 2\%  & 3.2\% & 0.14 \\
 & $10^6$          & 9  & 10 & 97\% & 3\%  & 2.8\% & 0.21 \\
\midrule
\multirow{4}{*}{Comp.}
 & $10^5$          & 2   & 2   & 60\% & 40\% & 3.6\% & 0.02 \\
 & $3{\times}10^5$ & 6   & 6   & 59\% & 41\% & 2.7\% & 1.84 \\
 & $10^6$          & 19  & 19  & 53\% & 47\% & 1.7\% & 2.46 \\
 & $10^7$          & 189 & 189 & 41\% & 59\% & 0.4\% & 4.94 \\
\bottomrule
\end{tabular}
\end{table}

Table~\ref{tab:cycle-anatomy} summarizes how MSC-CMA-ES spends its budget on
the $D = 10$ cells. The cycle count scales almost linearly with the budget---a
composition run grows from four cycles at $10^5$ evaluations to $378$ at
$10^7$---and the two configurations execute in near-equal numbers, since they
strictly alternate. Composition runs complete roughly twice as many cycles as
hybrid runs at a matched budget ($38$ versus $19$ at $10^6$), the hybrid loop
terminating earlier as its restarts converge and exhaust the reserve.
Although the two B/C parameter configurations  were tuned on disjoint function subsets
(see Section~\ref{sec:protocol}),
their resulting parameters admit a role reading that transfers across the
schedule: the composition-tuned C is a many-basin, small-population explorer,
and the unimodal/hybrid-tuned B a few-basin, large-population exploiter
(Table~\ref{tab:configs}). On composition landscapes the two roles engage
jointly---the refinement seed originates from both classes in comparable
proportion (Table~\ref{tab:cycle-anatomy})---whereas on hybrids it comes almost
exclusively from B.

The refinement seed and the refinement gain separate the two classes. On
hybrid functions the seed is almost always a B cycle ($89$--$98\%$ of runs)
and refinement is near-cosmetic, reducing the final error by under one order
of magnitude at every budget (Improv $\le 0.80$); the accuracy is produced by
the cycles, not the polish. On composition functions the seed is split between
the two configurations and refinement carries the accuracy: its contribution
grows monotonically with the budget, from essentially zero at $10^5$ to nearly
five orders of magnitude at $10^7$ (Improv $0.02\!\to\!4.94$). This pattern is
consistent with the class FBTC profile of Table~\ref{tab:class-summary}---deep-target
coverage is exactly what a multi-order refinement reduction produces---though
the present data establish the association, not its direction. The effect is
specific to the design envelope: at $D = 30$ the same refinement leaves the
incumbent essentially unchanged (median reduction below $0.15$ orders at both
$3\times10^5$ and $10^6$ evaluations), consistent with the sparser basin model of this section starving the seed of quality.

\FloatBarrier
\section{Conclusion}

MSC-CMA-ES demonstrates that restarts in CMA-ES can benefit from explicitly reusing sampled landscape structure. These gains are most evident on composition functions, where basin-aware seeding and redundant-visit suppression improve both error reduction and target coverage. The method is less effective on hybrid functions and loses its low-dimensional advantage near D=30 under the official evaluation budget. The latter is consistent with the finite-sample limits of nearest-better basin discovery. Future work should therefore focus on adaptive Phase-0 sample sizing, ablation studies of the basin-discovery, seeding, exclusion, and refinement components, and  hybrid-specific mechanisms.


\bibliographystyle{plainnat}

\begin{thebibliography}{99}

\bibitem{Akiba2019}
T.~Akiba, S.~Sano, T.~Yanase, T.~Ohta, and M.~Koyama.
\newblock Optuna: A next-generation hyperparameter optimization framework.
\newblock In \emph{Proc.\ ACM SIGKDD}, pages 2623--2631, 2019.
\newblock \doi{10.1145/3292500.3330701}.

\bibitem{AugerHansen2005}
A.~Auger and N.~Hansen.
\newblock A restart {CMA} evolution strategy with increasing population size.
\newblock In \emph{Proc.\ IEEE Congress on Evolutionary Computation (CEC)}, pages 1769--1776, 2005.
\newblock \doi{10.1109/CEC.2005.1554902}.

\bibitem{AwadEtAl2017}
N.~H. Awad, M.~Z. Ali, J.~J. Liang, B.~Y. Qu, and P.~N. Suganthan.
\newblock Problem definitions and evaluation criteria for the {CEC} 2017 special session on single objective real-parameter numerical optimization.
\newblock Technical report, Nanyang Technological University, 2017.


\bibitem{BrestEtAl2017}
J.~Brest, M.~S. Mau\v{c}ec, and B.~Bo\v{s}kovi\v{c}.
\newblock Single objective real-parameter optimization: Algorithm {jSO}.
\newblock In \emph{Proc.\ IEEE CEC}, pages 1311--1318, 2017.
\newblock \doi{10.1109/CEC.2017.7969456}.

\bibitem{Brest2020j2020}
J.~Brest, M.~S. Mau\v{c}ec, and B.~Bo\v{s}kovi\'{c}.
\newblock Differential evolution algorithm for single objective bound-constrained optimization: Algorithm {j2020}.
\newblock In \emph{Proc.\ IEEE CEC}, pages 1--8, 2020.
\newblock \doi{10.1109/CEC48606.2020.9185551}.


\bibitem{Hansen2001}
N.~Hansen and A.~Ostermeier.
\newblock Completely derandomized self-adaptation in evolution strategies.
\newblock \emph{Evolutionary Computation}, 9(2):159--195, 2001.
\newblock \doi{10.1162/106365601750190398}.

\bibitem{Hansen2009}
N.~Hansen.
\newblock Benchmarking a {BI-population CMA-ES} on the {BBOB-2009} function testbed.
\newblock In \emph{Proc.\ GECCO Companion}, pages 2389--2396, 2009.
\newblock \doi{10.1145/1570256.1570333}.

\bibitem{Hansen2019pycma}
N.~Hansen, Y.~Akimoto, and P.~Baudis.
\newblock {CMA-ES/pycma} on {GitHub}.
\newblock Zenodo, 2019. \url{https://github.com/CMA-ES/pycma}.
\newblock \doi{10.5281/zenodo.2559634}.

\bibitem{HansenCOCO2021}
N.~Hansen, A.~Auger, R.~Ros, O.~Mersmann, T.~Tu\v{s}ar, and D.~Brockhoff.
\newblock {COCO}: A platform for comparing continuous optimizers in a black-box setting.
\newblock \emph{Optimization Methods and Software}, 36(1):114--144, 2021.
\newblock \doi{10.1080/10556788.2020.1808977}.

\bibitem{HansenAnytime2022}
N.~Hansen, A.~Auger, D.~Brockhoff, and T.~Tu\v{s}ar.
\newblock Anytime performance assessment in blackbox optimization benchmarking.
\newblock \emph{IEEE Transactions on Evolutionary Computation}, 26(6):1293--1305, 2022.
\newblock \doi{10.1109/TEVC.2022.3210897}.


\bibitem{Liang2014cec2014}
J.~J. Liang, B.~Y. Qu, and P.~N. Suganthan.
\newblock Problem definitions and evaluation criteria for the {CEC} 2014 special session on single objective real-parameter numerical optimization.
\newblock Technical report, Zhengzhou University / Nanyang Technological University, 2013.



\bibitem{Minion2025}
K.~F. Muzakka, S.~M\"oller, and M.~Finsterbusch.
\newblock Minion: a high-performance derivative-free optimization library.
\newblock \url{https://github.com/khoirulmuzakka/Minion}, 2025.

\bibitem{ARRDE2023}
K.~F. Muzakka, A.~H. Shali, H.~Suhendar, S.~M\"oller, and M.~Finsterbusch.
\newblock Robust differential evolution via nonlinear population size reduction and adaptive restart: The {ARRDE} algorithm.
\newblock arXiv:2511.18429, 2025.
\newblock \doi{10.48550/arXiv.2511.18429}.

\bibitem{Preuss2010}
M.~Preuss.
\newblock Niching the {CMA-ES} via nearest-better clustering.
\newblock In \emph{Proceedings of the 12th annual conference companion on Genetic and evolutionary computation}, pages 1711--1718, 2010.
\newblock \doi{10.1145/1830761.1830793}.

\bibitem{Preuss2012}
M.~Preuss.
\newblock Improved topological niching for real-valued global optimization.
\newblock In \emph{European Conference on the Applications of Evolutionary Computation}, pages 386--395. Springer, 2012.
\newblock \doi{10.1007/978-3-642-29178-4_39}.


\bibitem{Stanovov2021nlshadersp}
V.~Stanovov, S.~Akhmedova, and E.~Semenkin.
\newblock {NL-SHADE-RSP} algorithm with adaptive archive and selective pressure for {CEC} 2021 numerical optimization.
\newblock In \emph{Proc.\ IEEE CEC}, pages 809--816, 2021.
\newblock \doi{10.1109/CEC45853.2021.9504959}.

\bibitem{Stanovov2024lsrtde}
V.~Stanovov and E.~Semenkin.
\newblock Success rate-based adaptive differential evolution {L-SRTDE} for {CEC} 2024 competition.
\newblock In \emph{Proc.\ IEEE CEC}, pages 1--8, 2024.
\newblock \doi{10.1109/CEC60901.2024.10611907}.

\bibitem{Tanabe2014}
R.~Tanabe and A.~S. Fukunaga.
\newblock Improving the search performance of {SHADE} using linear population size reduction.
\newblock In \emph{Proc.\ IEEE CEC}, pages 1658--1665, 2014.
\newblock \doi{10.1109/CEC.2014.6900380}.


\bibitem{TornViitanen1992}
A.~T\"{o}rn and S.~Viitanen.
\newblock Topographical global optimization.
\newblock In \emph{Recent Advances in Global Optimization}, pages 384--398. Princeton University Press, 1992.

\bibitem{Wang2022iohanalyzer}
H.~Wang, D.~Vermetten, F.~Ye, C.~Doerr, and T.~B\"ack.
\newblock {IOHanalyzer}: Detailed performance analyses for iterative optimization heuristics.
\newblock \emph{ACM Transactions on Evolutionary Learning and Optimization}, 2(3):1--29, 2022.
\newblock \doi{10.1145/3510426}.

\bibitem{SugathanCEC2020}
C.~T. Yue, K.~V. Price, P.~N. Suganthan, J.~J. Liang, M.~Z. Ali, B.~Y. Qu, N.~H. Awad, and P.~P. Biswas.
\newblock Problem definitions and evaluation criteria for the {CEC} 2020 special session and competition on single objective bound constrained numerical optimization.
\newblock Technical report, Zhengzhou University / Nanyang Technological University, 2019.

\end{thebibliography}

\end{document}